\documentclass{article} % For LaTeX2e
\usepackage{iclr2027_conference_arxiv,times}
\usepackage[T1]{fontenc}

\usepackage{amsmath,amsfonts,bm}

\def\eqref#1{equation~\ref{#1}}
\def\1{\bm{1}}

\newcommand{\test}{\mathcal{D_{\mathrm{test}}}}

\DeclareMathAlphabet{\mathsfit}{\encodingdefault}{\sfdefault}{m}{sl}
\SetMathAlphabet{\mathsfit}{bold}{\encodingdefault}{\sfdefault}{bx}{n}

\usepackage{hyperref}
\usepackage{url}
\makeatletter
\let\mycelium@citationlink\NAT@hyper@
\renewcommand{\NAT@hyper@}[1]{\mbox{\mycelium@citationlink{#1}}}
\makeatother
\usepackage{booktabs}
\usepackage{amsmath,amsthm}
\newtheorem{proposition}{Proposition}
\usepackage[ruled,vlined,linesnumbered]{algorithm2e}
\usepackage{cleveref}
\usepackage{booktabs,multirow}
\usepackage{graphicx}
\usepackage{tikz}
\usepackage{booktabs,tabularx}
\usepackage{multirow}
\usepackage{wrapfig}
\usepackage{longtable,needspace}

\title{Mycelium: A Generalizable Cross-Grid Multi-Task Model for Electrical Distribution Systems}

\author{%
Zhengyang Wei\textsuperscript{1,2}\quad
Shourya Bose\textsuperscript{2}\quad
Helgi Hilmarsson\textsuperscript{2}\quad
Elena Carnio\textsuperscript{2}\quad
Dhruv Suri\textsuperscript{2}\\
{\normalfont\textsuperscript{1}Stanford University\qquad\textsuperscript{2}Prav\={a}h}\\
{\normalfont\small\texttt{zywei@stanford.edu}}\\
{\normalfont\small\texttt{\{shourya.bose,helgi.hilmarsson,elena.carnio,dhruv.suri\}@pravah.com}}
}

\iclrfinalcopy % Show the author block in the arXiv version.
\begin{document}

\maketitle

\begin{abstract}
Electrical distribution grid operations require inference across heterogeneous networks from sparse, noisy, and incomplete time series measurements. In this work, we identify challenges and explore solutions towards a unified model that can perform diverse tasks grounded in the physics of the electric grid and generalize to unseen distribution networks.
We define a unified grid ontology that represents variable sized distribution networks as heterogeneous graphs while preserving native topology, asset types, and electrical relationships across networks.
We develop a physics based data simulation pipeline that combines reference and procedurally generated distribution networks with network reconfigurations, fault scenarios, and configurable sensing conditions.
We present \textit{Mycelium}, a heterogeneous graph transformer with structure aware communication edges and electrical reference features that encode network position and nominal phase orientation, together with task specific temporal readouts which generate per task outputs.
We train Mycelium on reference as well as synthetic grids, and study its generalization on benchmark networks completely excluded from training and validation. Mycelium is observed to outperform task specific neural baselines on most reported benchmark metrics.
Architectural ablations and the aforementioned studies reveal Mycelium's capability to learn representations of the underlying physics which serves to enhance cross-task performance, thereby addressing a significant challenge in unified grid models.

\end{abstract}

\section{Introduction}

% Why distribution system is important
% ML for grid
% XXXXX

% Most of the literature focus on transmission grid

% but distribution one is also very useful because of

Distribution networks supply homes, businesses, and critical services while integrating rooftop solar, batteries, electric vehicles, and electrified heating. Reliable operation amid increasingly variable demand and generation requires better monitoring and inference~\citep{cleenwerck2026distribution}. Machine learning offers a promising approach by learning physics based invariant patterns across networks and operating conditions~\citep{hamann2024foundation}. However, early efforts toward general purpose grid models have emphasized transmission systems, where measurements and power flow data are more standardized~\citep{hamann2024foundation,puech2026genco}.

% - partial metering, making classical methods in accurate
% - inaccurate topology, hard to fix
% - can directly help with everyday applications e.g. operations (outages), and through power flow (e.g. accurately reviewing connection requests, outage risks assesments, scenario planning, etc..  requires significant amount of simulations, requiring significant amounts of power flow solutions)
% - orders of magnitude more grids to use shared learning compared to transmission

Extending these capabilities to distribution networks presents distinct challenges. Transmission networks are generally better instrumented and commonly modeled using balanced single-phase equivalents~\citep{hamann2024foundation}. %Distribution networks instead combine radial or weakly meshed topologies with heterogeneous equipment, voltage levels, and phase configurations. Measurements are often sparse, asynchronous, and noisy, while customer-level telemetry may be unavailable because of privacy and ownership constraints. Network records can also become outdated as switching operations, customer connections, and distributed energy resources alter connectivity and operating conditions. ~\cite{fur2022post} illustrates these observational gaps, documenting limited central monitoring of low-voltage switch states. Consequently, topological quantities needed for electrical state estimation, such as switch states and customer phase assignments may themselves require inference. The volume of data generated by the power sector therefore does not automatically provide consistently labeled training examples. In developing countries, limited data infrastructure, data access, and analytical capacity further constrain the use of energy data~\citep{esmap2017energy}.
Distribution networks combine radial or weakly meshed topologies with diverse equipment, voltage levels, and phase configurations, and distribution measurements are often sparse, asynchronous, noisy, or restricted by privacy and ownership. Evolving connectivity and operating conditions can also render network records outdated. Limited central monitoring of low voltage switches~\citep{fur2022post} illustrates why switch states and customer phase assignments may require inference alongside electrical states. Thus, abundant power sector data do not necessarily provide consistently labeled training examples. In developing countries, limited data infrastructure, access, and analytical capacity further constrain their use~\citep{esmap2017energy}.
%
% Why we need cross-grid model?
Existing learning based approaches address important distribution grid tasks, including state estimation, phase attribution, and fault diagnosis~\citep{duan2019deep,dande2025consumer,chanda2024heterogeneous}. However, a model developed for a particular network can rely on patterns specific to its topology, measurement configuration, and operating distribution. Maintaining separate models for every network and task would multiply data, computation, and engineering requirements as networks evolve. This motivates the use of shared models that remain useful across heterogeneous networks and support multiple inference tasks. The central question is whether a single model can generalize to unseen distribution networks under sparse, noisy, and incomplete measurements.

We introduce \emph{Mycelium}, a shared heterogeneous graph model for state estimation, phase attribution, switch state inference, and fault diagnosis. Mycelium uses heterogeneous equipment representations, a domain specific graph representation with virtual communication edges, and inductive biases constraining electrical quantity outputs to support learning across topologically and electrically diverse networks. We combine this architecture with a procedure for synthetically generating diverse yet representative distribution grids to study the intersection of topological generalization and task generalization under a single model. Our main contributions are:

\begin{itemize}
\item \textbf{Domain informed architecture.}
We introduce a heterogeneous graph architecture combining equipment aware representations, communication structures beyond the grid's own topology, and reference setpoints for quantities such as voltage angles. Ablation studies are performed to quantify their effects across tasks.

\item \textbf{Multi-task capability with sparse measurements.}
Mycelium supports the tasks of state estimation, phase attribution, switch state inference, and fault diagnosis from sparse, noisy, and incomplete measurements. Comparisons with single task training and finetuning characterize the benefits and tradeoffs of sharing.

\item \textbf{Out-of-distribution generalization.}
We evaluate Mycelium on 52 unseen test grids and a frozen benchmark of five grids from source families excluded from training and validation. The results show promising cross-grid generalization, which we further characterize through ablations among architectural choices, task outputs, and grid size.

\item \textbf{Physics grounded data generation and augmentation.}
We develop a simulation pipeline combining curated reference networks and procedurally generated synthetic networks under network reconfigurations, faults, and varied sensing conditions. Scaling experiments quantify the benefits and task dependent tradeoffs of synthetic augmentation on unseen grids.
\end{itemize}

\section{Related Work}

\paragraph{Topology aware and cross-grid graph learning.}

GridFM \citep{hamann2024foundation} proposes initial steps towards a grid foundation model, introducing the idea of masked reconstruction over diverse grid topologies and operating conditions, followed by finetuned adaptation. \cite{zhu2025gnns} propose a physics informed self supervised graph transformer with reusable initialization across power flow (PF), optimal PF (OPF), contingency analysis, and outage prediction, but demonstrate neither cross-grid transfer nor a single shared model across grids.
MxGPS \citep{papaioannou2026mxgps} shows that joint PF and state estimation training can improve cross-grid robustness, but at the cost of in-distribution accuracy and with only two related regression tasks on small, balanced transmission networks. Finally, \citet{huang2026unified} show that a frozen encoder can provide a reusable representation interface across unseen transmission grids and heterogeneous computational tasks. Their downstream computations nevertheless rely on decoders trained independently per task.

\paragraph{Learning for distribution systems.}
\citet{zamzam2020physics} exploit known grid topology to sparsify neural estimators. DSS$^{2}$~\citep{habib2023deep} combines heterogeneous graphs with weighted least squares (WLS) residuals and power flow constraints to learn the full state without clean state labels, while F-PINN~\citep{azam2026physics} uses factor graphs and augmented Lagrangian training to enforce physical constraints. For phase identification, \citet{dande2025consumer} train GNNs on network specific power flow simulations, assuming known customer mapping to upstream transformers. Switch state inference has progressed from CNN classification of simulated micro phasor measurement unit ($\mu$PMU) time series~\citep{duan2019deep} to topology aware GNNs~\citep{madbhavi2023distribution}. Alternatively DGS-GNN~\citep{ebtia2024power} uses GraphSAGE~\citep{hamilton2017inductive} embeddings of measurements alongside physical topology to recover switch states. For fault diagnosis, \citet{chanda2024heterogeneous} jointly predict various attributes of a given fault event using graph convolutional networks (GCNs), while \citet{karabulut2026assessing} study spatiotemporal fault localization under distributed energy resources (DER) shifts, finding asymmetric generalization across penetration regimes. These studies suffer from specificity to grid topologies, leaving the problem of cross-grid generalization understudied.

\paragraph{Generalist learning for transmission systems.}
Recent work extends task specific surrogates toward cross-grid and generalist models. CANOS~\citep{piloto2024canos} approximates AC OPF on networks of up to 10,000 buses under $N-1$ outages, but relies on per grid training. PF$\Delta$~\citep{rivera2026pf} benchmarks power flow across six large scale transmission systems, but demonstrates limitations in electrical feasibility and cross-grid transfer. LUMINA~\citep{li2026lumina} performs a study in zero-shot AC OPF cross-grid generalization through pretraining across multiple topologies, but only demonstrates generalization across up to two unseen grids. GridSFM~\citep{yang2026gridsfm} pretrains across approximately 200 networks for AC OPF and feasibility screening, although cross-grid generalization is shown to require finetuning. GENCO~\citep{puech2026genco} provides a common architecture for power flow, OPF, and state estimation, but uses separate task/grid checkpoints and evaluates cross-grid pretraining only for power flow. The aforementioned works show a lack of cross-grid generalization across multiple tasks, especially under temporally varying sparse measurements.

% \paragraph{Multi-task representation learning for physical systems.}

% For distribution system

% For transmission system:
% Gridsfm\
% genco
% canos
% pfdelta

\section{Task Overview and Mycelium Architecture}
% We study four tasks:

% \begin{itemize}
% \item \textbf{State estimation:} reconstructing per-bus, per-phase voltage-magnitude and angle trajectories at both observed and unobserved buses;
% \item \textbf{Phase attribution:} identifying the A/B/C phase connection of each single-phase customer served by a three-phase transformer;
% \item \textbf{switch state classification:} inferring whether each switch is open or closed from electrical measurements alone; and
% \item \textbf{Fault classification and localization:} determining whether a fault occurred during the observation window and, if so, predicting its type and affected element.
% \end{itemize}

% We restrict the scope of this paper to inference from (possibly) sparse measurements, assuming no ability to control generation, grid equipment, or DER assets. We provide a mathematical formulation of these tasks in Appendix \ref{sec:task_overview}.
\subsection{The Four Tasks}
We infer electrical states, connection attributes, and fault events from a network graph and partial measurements over $T=24$ hourly timesteps. Table~\ref{tab:task_summary} summarizes the four tasks: state estimation (SE) predicts continuous voltage trajectories. Phase attribution (PA) and switch state inference (SW) classify consumer connections and branch states, respectively, which remain fixed within the window. Switch locations are known, while their open/closed states are predicted via edge classification. Fault diagnosis (F) predicts normal operation or one of five fault classes and, for faulted samples, the affected equipment under the assumption that the fault onset occurs at the final observed timestep. We perform inference from available observations without assuming control of generation or grid equipment. Formal definitions are given in Appendix~\ref{app:task_overview}.

\label{sec:model}

% Mycelium uses one parameter set for four inference tasks across variable-size distribution networks. Given a structural graph and partially observed measurements over $T=24$ hourly timesteps, it applies shared spatial message passing followed by task-specific temporal readouts. The outputs comprise per-timestep voltage states and window-level phase, switch state, and fault predictions. Figure~\ref{fig:grid_ontology} summarizes the architecture.

\begin{figure}[t]
\centering
\includegraphics[width=0.98\linewidth]{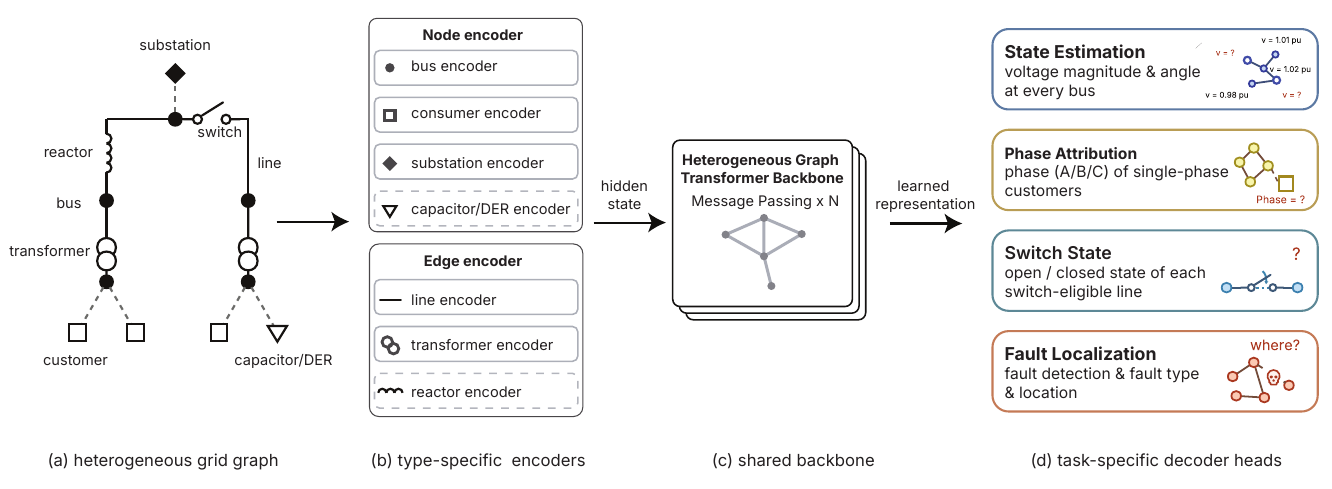}
\caption{Mycelium architecture. (a) A heterogeneous graph models network components and electrical relations. (b) Type specific encoders embed attributes and temporal measurements. (c) A shared heterogeneous graph transformer performs spatial message passing. (d) Task specific temporal heads support SE, PA, SW, and fault diagnosis.}
\label{fig:grid_ontology}
\end{figure}

\subsection{Mycelium's Graph Representation and Inputs}
\label{sec:grid_ontology}
Mycelium encodes the structural graph and masked measurements, applies shared spatial message passing at each timestep, and uses task specific temporal readouts to produce the four outputs. Figure~\ref{fig:grid_ontology} summarizes the architecture.
Each network is defined as a heterogeneous graph $\mathcal{G}=(\mathcal{V},\mathcal{E})$ with bus, consumer, and substation nodes, together with capacitor and DER nodes when present. Lines, transformers, and reactors form distinct bus--bus relations. Service and source links attach consumers and substations to their buses. All switch eligible line edges remain present regardless of operating state.
The network components satisfy the AC power flow equations~\citep{grainger1994power}, given as
% \begin{equation}
$
    \mathbf{f}\!\left(
    \mathbf{V}_t,\boldsymbol{\Theta}_t,\mathbf{p}_t,\mathbf{q}_t;
    \mathcal{G},\mathbf{c}_t
    \right)=\mathbf{0},
    \label{eq:abstract_power_flow}
$
% \end{equation}
where $\mathbf{V}_t$ and $\boldsymbol{\Theta}_t$ are per phase bus voltage magnitudes and angles, $\mathbf{p}_t$ and $\mathbf{q}_t$ are net active and reactive power injections, and $\mathbf{c}_t$ collects equipment parameters and operating configurations. Equipment nameplate ratings and static electrical parameters are assumed available, while switch open/closed states are excluded from edge inputs. Inputs combine equipments' static attributes with available voltage, power, current, and feeder head measurements. Observability masks distinguish zero filled missing readings from observed zeros. Separate supervision masks identify valid prediction targets.
Type specific node and edge encoders map these inputs into a common hidden dimension while retaining the temporal axis. Their parameters are shared across entities of the same type and across networks. Appendix~\ref{app:input_schema} describes the input representation and masks.

\subsection{Distribution Grid Inductive Biases}
% Umbrella: distribution networks are deep, radial and weakly observed. Both
% mechanisms below attack the resulting long-range information problem, by
% changing the REPRESENTATION rather than the capacity. Cite the ablation
% table once here; per-mechanism numbers go in the subsections.

\begin{wrapfigure}{r}{0.35\textwidth}
\centering
\includegraphics[width=\linewidth]{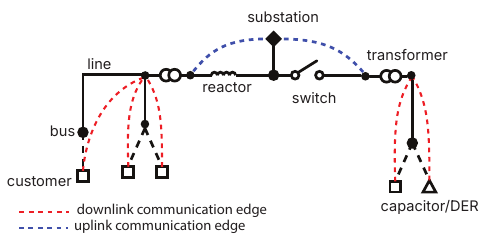}
\caption{Mycelium’s message passing graph inherits the base grid topology and adds communication edges. Blue dashed edges connect the substation to transformer primary buses, and red dashed edges connect customers to their serving transformer’s secondary bus.}
\label{fig}
\end{wrapfigure}
% \subsubsection{Communication Graph}
% \label{sec:communication_graph}
\paragraph{Communication Edges.}
Long radial paths can require many message passing layers to exchange information between customers and the substation. We shorten these paths with two typed bidirectional communication edges, illustrated in Figure~\ref{fig}: \emph{uplinks} connect the substation to each distinct distribution transformer primary bus, and \emph{downlinks} connect customers to their serving transformer's secondary bus.
Together with the physical transformer relation, these bidirectional shortcuts place each customer with a valid transformer association within three message passing hops of the substation. The connections are derived from the available topology and customer registry and introduce no additional measurement features. Relation specific parameters distinguish the shortcuts from physical electrical connections.

% We motivate the choice of these communication edges through their role in
% state estimation. Reconstructing bus voltages from partial observations can
% require measurements beyond a predictor's physical message passing
% neighborhood. The following result shows how communication edges can reduce
% the resulting prediction error using a linear message passing learner.
We analyze voltage prediction with a linear learner using power-of-two
ancestor links. This idealized communication graph differs from Mycelium's
equipment-based uplinks and downlinks.
On a radial grid of physical depth $h\ge1$, let $\mathbf V$ contain bus voltage
magnitudes and $\mathbf p,\mathbf q$ the active and reactive injections.
Assume the linear distribution flow (LinDistFlow) model \citep{deka2024learning},
\begin{equation}
 \mathbf V^{\odot2}=2\mathbf R\mathbf p+2\mathbf X\mathbf q+V_0^2\mathbf1.
 \label{eq:communication_lindistflow}
\end{equation}
Here squaring is elementwise, $V_0$ is fixed, and $\mathbf R,\mathbf X$
encode grid topology and powerline resistances and reactances.
Suppose the true injections have joint covariance $\boldsymbol{\Sigma}\succ0$.
Each bus receives its own injection measurements,
with additive zero-mean noise independent of the injections and joint noise
covariance $\tau\boldsymbol{\Sigma}$, where $\tau\ge0$ is the ratio of sensing variance
to ground truth variance.
With a slight abuse of notation, $\mathbf p,\mathbf q$ also denote the
sensed learner inputs.
We study the prediction task $(\mathbf p,\mathbf q)\mapsto\mathbf V^{\odot2}$.
Let $\mathcal D$ denote the joint probability distribution of injections
and measurement noise described above. For snapshots $s\sim\mathcal D$,
denote predicted squared magnitudes by $\widehat{\mathbf V^{\odot2}}_{\mathbf{w}}(s)$ and train weights $\mathbf{w}$ on
\begin{equation}
 \mathcal L(\mathbf{w})=\frac12\,\mathbb{E}_{s\sim\mathcal D}
 \left[\left\|
 \widehat{\mathbf V^{\odot2}}_{\mathbf{w}}(s)-\mathbf V(s)^{\odot2}
 \right\|_2^2\right].
 \label{eq:communication_loss}
\end{equation}
\begin{proposition}[Exact communication and voltage prediction]\label{prop:communication_learning}
Under the preceding assumptions, there exists a linear learner $H$ for this task
that computes each prediction in $2\lceil\log_2 h\rceil$ message passing rounds
using physical grid edges and added bidirectional ancestor edges.
Let $\lambda_j>0$ be its loss-Hessian eigenvalues and $c_j$ its initial
parameter errors relative to the optimum in an orthonormal eigenbasis. Use the constant gradient descent (GD) step size
$\eta'=\eta/(1+\tau)$, where $\eta$ is the step size for noiseless inputs.
After $n\ge1$ full GD steps from zero initialization with
$0<\eta'<2/\max_j\lambda_j$, its error
$\mathcal E_{\rm H}(n)=2\mathcal L(\mathbf w^{(n)})$ satisfies, for every integer
$\ell\ge2\lceil\log_2 h\rceil$,
\begin{equation}
 \mathcal E_{\rm P}^\star(\ell)-\mathcal E_{\rm H}(n)
 =\frac{C_\ell}{1+\tau}
  -\sum_j\lambda_j(1-\eta'\lambda_j)^{2n}c_j^2.
 \label{eq:communication_prediction_gap}
\end{equation}
Here $\mathcal E_{\rm P}^\star(\ell)$ is the minimum error among linear predictors
using only physical grid edges and at most $\ell$ message passing rounds per prediction.
\end{proposition}

Here $C_\ell$ is the squared voltage variation not linearly explained by
injections within $\ell$ physical hops. When $C_\ell>0$, the gap becomes positive
after finitely many steps and tends to $C_\ell/(1+\tau)$.
Appendix~\ref{app:communication_theory} gives the Schur-complement expression,
learner construction, and proof.
%
% Mycelium's equipment based uplinks and downlinks, while using only a subset
% of the possible hierarchical connections, make such information available
% to the shared representations used by all four tasks.
% The guarantee requires injection measurements at every bus and the stated
% noise model. Benefits with the practical graph and missing measurements
% are evaluated empirically.
The result assumes exact LinDistFlow, injection coverage at every bus, clean voltage targets, and population GD. Mycelium's nonlinear model, partial measurements, and AdamW training are evaluated empirically.
% Sparse, tree-structure network topologies can require many message-passing steps to exchange information between consumers and the substation, some are even more than 50.
% We augment the physical graph with two typed shortcut relations:
% \begin{itemize}
%     \item \emph{uplinks} connect the substation to each distinct distribution-transformer primary bus,
%     \item \emph{downlinks} connect consumers to their serving transformer's secondary bus.
% \end{itemize}

% These relations are derived from the available topology and consumer registry and introduce no additional measurement features.
% Bidirectional message passing provides the path
% \[
% \text{substation}
% \leftrightarrow
% \text{primary bus}
% \leftrightarrow
% \text{secondary bus}
% \leftrightarrow
% \text{consumer},
% \]
% placing consumers with valid transformer associations within three hops of the feeder head. Separate relation-specific parameters allow the model to distinguish these communication shortcuts from physical electrical connections.

% \begin{figure}[t]
% \centering
% \includegraphics[width=0.6\linewidth]{fig/comm.pdf}
% \caption{XXX.}
% \label{fig:grid_ontology}
% \end{figure}

% \subsubsection{Nominal-Angle Reference}
% \label{sec:nominal_angle_reference}
\paragraph{Nominal Angle Reference.}
Voltage angles contain predictable offsets from source phase sequences and transformer connections. We derive a static nominal angle $\theta^{\mathrm{nom}}_{i,p}$ for each (bus, phase) pair from equipment metadata, without solved voltage trajectories. Observed angles and targets are expressed relative to this reference as
$\Delta\theta_{i,p,t} = \operatorname{wrap}(\theta_{i,p,t} - \theta^{\mathrm{nom}}_{i,p})$,
where $\operatorname{wrap}$ maps angles to $[-\pi,\pi)$, and predictions are
converted back as
$\widehat{\theta}_{i,p,t} = \operatorname{wrap}(\theta^{\mathrm{nom}}_{i,p} + \widehat{\Delta\theta}_{i,p,t})$.
This parameterization focuses regression on deviations from nominal phase
structure. Appendix~\ref{app:nominal_angle_reference} describes the reference construction.

\subsection{Shared Backbone and Task Readouts}
\label{sec:hgt_backbone}
\label{sec:task_readouts}

We use a heterogeneous graph Transformer~\citep{hu2020heterogeneous} with relation specific attention over physical, attachment, and communication edges.
At layer $\ell$, each node aggregates messages within each relation and sums their contributions:
\begin{equation}
\mathbf{m}_{i,t}^{(\ell)}
=
\sum_r \mathop{\bigg\Vert}_{k=1}^{H}
\sum_{j\in\mathcal{N}_r(i)}
\alpha_{ji,t}^{(\ell,r,k)}
\mathbf{v}_{ji,t}^{(\ell,r,k)}.
\end{equation}

% Here, $\mathcal{N}r(i)$ denotes incoming neighbors under relation $r$, $k$ indexes the $H$ attention heads, and $\Vert$ denotes concatenation. The weights $\alpha{ji,t}^{(\ell,r,k)}$ scale messages $\mathbf{v}_{ji,t}^{(\ell,r,k)}$ from $j$ to $i$; outputs are concatenated across heads and summed across relations.
Here, $\mathcal{N}_r(i)$ contains incoming neighbors under relation $r$, $k$ indexes the $H$ attention heads, and $\Vert$ denotes concatenation. The coefficient $\alpha_{ji,t}^{(\ell,r,k)}$ weights the message $\mathbf{v}_{ji,t}^{(\ell,r,k)}$ from $j$ to $i$. Head outputs are concatenated and relation outputs summed.
Physical edge embeddings condition both attention scores and message values. Each layer combines attention with pre normalization and an FFN specific to the node type, with residual connections around both. Spatial attention operates independently at each timestep with shared parameters, preserving $\mathbf{H}^{(\ell)}\in\mathbb{R}^{|\mathcal{V}|\times T\times d}$ without temporal pooling. The model uses eight layers, hidden dimension $128$, four attention heads, and approximately $12.2$M parameters. Appendix~\ref{app:hgt_backbone} provides further details.
For state estimation, a causal dilated temporal convolutional network~\citep{bai2018empirical} processes each bus sequence. A shared linear head predicts per phase voltage magnitudes and nominal angle residuals, with absolute angles recovered as described in Appendix~\ref{app:nominal_angle_reference}. For phase attribution, mask aware temporal attention pools consumer, attached bus, and network representations for three-way classification. For switch inference, the readout combines attention weighted, mean, and maximum temporal summaries of endpoint and branch representations.
For fault diagnosis, a temporal readout compares the final representation with an attended history using signed and absolute differences. Network and candidate readouts produce $\mathbf{z}_{\mathcal{G}}^{\mathrm{F}}$ and $\mathbf{z}_{c}^{\mathrm{F}}$.
% A shared scorer evaluates the network's bus, line, and transformer candidates:
% \begin{equation}
% s_c=f_{\mathrm{loc}}\!\left(
% \left[\mathbf{z}_{c}^{\mathrm{F}}\Vert\mathbf{z}_{\mathcal{G}}^{\mathrm{F}}\right]\right),
% \qquad c\in\mathcal{C}(\mathcal{G}).
% \label{eq:fault_candidate_score}
% \end{equation}
A shared scorer evaluates the network's bus, line, and transformer candidates,
$s_c = f_{\mathrm{loc}}([\mathbf{z}_{c}^{\mathrm{F}} \Vert \mathbf{z}_{\mathcal{G}}^{\mathrm{F}}])$
for $c \in \mathcal{C}(\mathcal{G})$.
Scoring the available candidates accommodates different network sizes without a fixed location vocabulary. A separate classifier combines network context and maximum pooled candidate evidence to distinguish normal operation and five fault classes. Appendix~\ref{app:task_readouts} describes all readouts.

\subsection{Multi-task Training}
\label{sec:multitask_objective}

We jointly optimize the backbone and task heads with unit coefficients for the four task losses.
% \begin{equation}
% \mathcal{L}=\mathcal{L}_{\mathrm{SE}}+
% \mathcal{L}_{\mathrm{phase}}+
% \mathcal{L}_{\mathrm{switch}}+
% \mathcal{L}_{\mathrm{fault}}.
% \label{eq:multitask_loss}
% \end{equation}
State estimation uses scaled Smooth-$L_1$ errors for voltage magnitude and wrapped angle. Phase attribution and switch inference use cross entropy. The fault loss combines class weighted six-class classification and equipment localization. For entity level losses, we average over valid targets within each network window and then over applicable windows, preventing target count alone from giving larger networks greater weight. Appendix~\ref{app:multitask_objective} describes the detailed losses and masking.

\section{Data and Evaluation Protocol}
\subsection{Data}
\paragraph{Reference and Synthetic Network Corpus}
The reference corpus combines publicly available models from NREL SMART-DS~\citep{OEDI_Dataset_2981}, EPRI, the IEEE test feeder suite~\citep{schneider2017analytic}, D-Suite~\citep{Deakin2025}, Iowa State~\citep{bu2019time}, and European utilities~\citep{taye2024set}. Appendix~\ref{app:reference_data_sources} describes these sources.
We simulate the network models in OpenDSS~\citep{dugan2011open} using convergence checked three-phase unbalanced power flow. Operating cases include transformer reassignment, consumer phase changes, coordinated switching, and faults at buses, lines, and transformer terminals. Seven sensor policies vary measurement coverage, noise, and missingness (Appendix~\ref{app:sensor_policies}). The reference corpus contains 98{,}799 electrical cases and 478{,}891 retained sensor policy samples after filtering.
%
% \paragraph{Synthetic network Corpus}
%
We also generate diverse feeders from street and building geometry, grouping customers into transformer service areas and routing primary and secondary networks. Equipment parameters and planning conventions are sampled from the training reference networks, with 24-hour demand, reactive power, and PV profiles defining operating conditions. Geography, equipment conventions, load climate, and solar region are sampled separately. OpenDSS~\citep{dugan2011open} checks daily operation and high load and high PV stress cases. Bounded repairs adjust phases, taps, reactive support, and equipment sizing. Networks still violating convergence or operating limits are rejected. Appendix~\ref{app:synthetic_generation} details generation and acceptance criteria.

\paragraph{Data Split}
\label{sec:data_split}
We partition reference networks before generating variants, temporal windows, or sensor policy realizations. All derived samples retain their source network's assignment. Each split includes networks exhibiting three task relevant characteristics: meshed topology, single-phase customers served by three-phase distribution transformers, and multiple distribution transformers. Table~\ref{tab:data_split} summarizes the network disjoint assignments.
The test split contains 52 unseen networks and 10{,}038 samples, largely from source families represented in training. The frozen benchmark contains 1{,}043 simulated samples on five utility derived network models from source families excluded from training and validation. The two splits assess generalization across networks and source families, respectively. Checkpoint selection uses validation loss. Because benchmark scenarios and sensor policies repeatedly sample the same five networks, the sample count does not represent 1{,}043 independent network structures.

\begin{table}[t]
\centering
\caption{Network level assignments.}
\label{tab:data_split}
\resizebox{\linewidth}{!}{%
\begin{tabular}{@{}ll@{}}
\toprule
Split & Networks \\
\midrule
Train & SMART-DS AUS/SFO ($70\%$ each); registered IEEE feeder; EPRI \texttt{ckt5}, \texttt{ckt24}; D-Suite \texttt{spd\_r/s/u} \\
Validation & SMART-DS AUS/SFO ($15\%$ each); EPRI \texttt{ckt7}; D-Suite \texttt{spm\_s} \\
Test & SMART-DS AUS/SFO ($15\%$ each) and all GSO; D-Suite \texttt{spm\_r/u} \\
Benchmark & Iowa 240- and 296-bus; European industrial, rural, and urban \\
\bottomrule
\end{tabular}%
}
\end{table}

\subsection{Metrics}
\label{sec:metrics}

\textbf{State estimation.}
Over valid (bus, timestep, phase) entries, we report voltage magnitude MAE in per unit ({p.u.}) and volts (${v_{\mathrm{mag}}}$), angle MAE in degrees (${v_{\mathrm{ang}}}$), and magnitude MAPE. Per unit errors use each bus's voltage base, so the two units can rank models differently across networks.
\textbf{Phase attribution.}
We report three-class top-1 accuracy for single-phase customers on multiphase transformers with at least one observed reading.
\textbf{Switch state.}
We report precision, recall, and F1 over applicable switchable lines, excluding switches whose two endpoints are marked de-energized. OPEN is the positive class, and predictions use a fixed zero logit threshold.
\textbf{Fault diagnosis.}
We report detection F1 ({Det.\ F1}) over normal and faulted samples in each split. On faulted samples, we report five-class type accuracy ({class$_5$}), exact equipment localization accuracy ({Loc.}), and localization within $k$ graph hops ({Hop@$k$}).

\subsection{Implementation Details}

% To assess whether a single shared model can provide competitive performance across heterogeneous distribution grid inference tasks, we compare Mycelium with task-specific baselines:

% \begin{itemize}
%     \item \textbf{State estimation:} weighted least squares (WLS) and a CANOS-based neural estimator~\citep{piloto2024canos} adapted to the same measurement inputs.
%     \item \textbf{Phase attribution:} correlation-based  matching and GraphSAGE-CPI~\citep{dande2025consumer}.
%     \item \textbf{switch state inference:}
%     %GNN-TI \citep{madbhavi2023distribution} and
%     DGS-GNN~\citep{ebtia2024power}.
%     \item \textbf{Fault diagnosis:}
%     %the multi-task GCN of \citet{chanda2024heterogeneous} for joint fault-attribute prediction, and
%     STGATv2 \citep{karabulut2026assessing}.
% \end{itemize}
%
%\subsection{Implementation details}
We train Mycelium on eight 80\,GB NVIDIA A100 GPUs using FP32,
activation checkpointing, and a global batch size of 64. The backbone has eight Transformer layers with pre normalization, hidden width 128, four attention heads, feedforward width 256, and dropout 0.1.
Training uses 12{,}800 samples per epoch for up to 100 epochs, with AdamW~\citep{loshchilov2017decoupled}, learning rate $10^{-3}$, cosine decay, weight decay $10^{-2}$, and gradient clipping at 1.0. We retain the checkpoint with the lowest validation loss and stop after ten epochs without improvement. The reference configuration combines reference and synthetic networks with a 10{,}000-bus training cap. Appendix~\ref{app:training_memory} details memory optimization.

\section{Experiments}

\subsection{Baseline Comparisons}

\begin{table*}[t]
\centering
\caption{Mycelium and task-specific neural baselines on the held-out test split and frozen benchmark. Dashes indicate tasks a baseline does not model. Lower is better for state estimation errors and higher is better otherwise. Best values within each split are bold.}
\label{tab:baselines}
\resizebox{\textwidth}{!}{%
\begin{tabular}{lllcccccccccccccc}
\toprule
& & & \multicolumn{4}{c}{\textbf{State estimation}} & \textbf{Phase} &
\multicolumn{3}{c}{\textbf{Switch}} & \multicolumn{6}{c}{\textbf{Fault}} \\
\cmidrule(lr){4-7}\cmidrule(lr){8-8}\cmidrule(lr){9-11}\cmidrule(lr){12-17}
Split & Baseline & Task &
p.u. $\downarrow$ &
$v_{\mathrm{mag}}$ (V) $\downarrow$ &
$v_{\mathrm{ang}}$ (deg) $\downarrow$ &
$v_{\mathrm{mag}}$ MAPE (\%) $\downarrow$ &
acc $\uparrow$ &
prec $\uparrow$ & rec $\uparrow$ & F1 $\uparrow$ &
det F1 $\uparrow$ & class$_5$ $\uparrow$ & loc $\uparrow$ &
hop@1 $\uparrow$ & hop@2 $\uparrow$ & hop@3 $\uparrow$ \\
\midrule
\multirow{5}{*}{Test}
& CANOS~\citep{piloto2024canos} & State est. &
  \textbf{0.00221} & \textbf{2.86} & \textbf{0.576} & \textbf{0.136} &
  - & - & - & - & - & - & - & - & - & - \\
& GraphSAGE-CPI~\citep{dande2025consumer} & Phase &
  - & - & - & - &
  0.3276 &
  - & - & - & - & - & - & - & - & - \\
& DGS-GNN~\citep{ebtia2024power} & Switch &
  - & - & - & - & - &
  0.0807 & 0.2399 & 0.1207 &
  - & - & - & - & - & - \\
& STGATv2~\citep{karabulut2026assessing} & Fault &
  - & - & - & - & - & - & - & - &
 - & - & 0.1349 & 0.2162 & 0.2683 & 0.3107 \\
\cmidrule(lr){2-17}
& Mycelium 10k & All four &
  0.00345 & 5.99 & 1.107 & 0.234 &
  \textbf{0.5713} &
  \textbf{0.9960} & \textbf{0.8137} & \textbf{0.8957} &
  \textbf{1.0000} & \textbf{0.4840} & \textbf{0.1710} & \textbf{0.2369} & \textbf{0.2999} & \textbf{0.3314} \\
\midrule
\multirow{5}{*}{Benchmark}
& CANOS~\citep{piloto2024canos} & State est. &
  0.01754 & \textbf{9.50} & 3.254 & 5.591 &
  - & - & - & - & - & - & - & - & - & - \\
& GraphSAGE-CPI~\citep{dande2025consumer} & Phase &
  - & - & - & - &
  0.3369 &
  - & - & - & - & - & - & - & - & - \\
& DGS-GNN~\citep{ebtia2024power} & Switch &
  - & - & - & - & - &
  0.4737 & 0.4114 & 0.4404 &
  - & - & - & - & - & - \\
& STGATv2~\citep{karabulut2026assessing} & Fault &
  - & - & - & - & - & - & - & - &
  - & - & 0.1177 & 0.1772 & 0.1825 & 0.2262 \\
\cmidrule(lr){2-17}
& Mycelium 10k & All four &
  \textbf{0.01267} & 81.28 & \textbf{1.635} & \textbf{5.425} &
  \textbf{0.4001} &
  \textbf{0.9553} & \textbf{0.7198} & \textbf{0.8210} &
  \textbf{1.0000} & \textbf{0.7209} & \textbf{0.1958} & \textbf{0.3175} & \textbf{0.3558} & \textbf{0.4114} \\
\bottomrule
\end{tabular}%
}
\end{table*}

We compare Mycelium with task specific baselines: WLS and a CANOS based estimator~\citep{piloto2024canos} adapted to the same inputs for state estimation, correlation based matching and GraphSAGE-CPI~\citep{dande2025consumer} for phase attribution, DGS-GNN~\citep{ebtia2024power} for switch state inference, and STGATv2~\citep{karabulut2026assessing} for fault diagnosis.

Table~\ref{tab:baselines} compares task-specific neural baselines with Mycelium 10k, our reference configuration trained on the full reference training set and synthetic network corpus with a maximum training network size of 10{,}000 buses. This size cap applies only during training. Evaluation networks are not truncated. CANOS achieves lower errors on all four state estimation metrics on the test split. On the frozen benchmark, Mycelium has lower per-unit voltage error, angle error, and magnitude MAPE, but higher voltage error in volts: 81.28V versus 9.50V. Across both splits, Mycelium improves phase accuracy over GraphSAGE-CPI, switch state F1 over DGS-GNN, and exact fault localization over STGATv2. We report WLS state estimation and correlation-based phase matching in Appendix~\ref{sec:further-res}. Together, these comparisons show that one Mycelium checkpoint supports all four tasks with competitive accuracy on feeders from unseen source families, without the privileged operating information supplied to WLS.
Appendix~\ref{app:per-network} reports per-network and network macro results, revealing substantial variation across benchmark networks.

\subsection{Ablation Studies}
% same number of communication edges, but randomly distributed between nodes\\
% model component

% We evaluate which architectural choices contribute to cross grid performance. Each ablation changes one component while retaining the training data and optimization protocol of the full model.

\begin{table*}[t]
\centering
\caption{Architectural ablations on the test split and benchmark:
(a) communication edges and (b) nominal angle references. Random rewiring preserves edge counts and endpoint types per relation and network. Training data and optimization remain fixed.
Best values are bold.}
\label{tab:ablations}
\small
\setlength{\tabcolsep}{3pt}
\renewcommand{\arraystretch}{1.1}
\resizebox{\textwidth}{!}{%
\begin{tabular}{cl*{14}{r}}
\toprule
& & \multicolumn{4}{c}{State estimation}
& Phase & \multicolumn{3}{c}{Switch} & \multicolumn{6}{c}{Fault} \\
\cmidrule(lr){3-6}\cmidrule(lr){7-7}
\cmidrule(lr){8-10}\cmidrule(lr){11-16}
Split & Configuration
& p.u. $\downarrow$ & $v_{\mathrm{mag}}$ (V) $\downarrow$
& $v_{\mathrm{ang}}$ (deg) $\downarrow$ & MAPE (\%) $\downarrow$
& Acc. $\uparrow$ & Prec. $\uparrow$ & Rec. $\uparrow$ & F1 $\uparrow$
& Det. F1 $\uparrow$ & class$_5$ $\uparrow$ & Loc. $\uparrow$
& Hop@1 $\uparrow$ & Hop@2 $\uparrow$ & Hop@3 $\uparrow$ \\
\midrule
\multirow{3}{*}{\shortstack[l]{(a) Test}} & No communication & 0.00370 & \textbf{4.99} & 1.286 & 0.240 & \textbf{0.6611} & 0.1950 & \textbf{0.9800} & \underline{0.3253} & \textbf{1.0000} & 0.4540 & 0.1819 & \textbf{0.2743} & \textbf{0.3464} & \textbf{0.3822} \\
 & Random communication & \textbf{0.00316} & 5.24 & \textbf{1.083} & \textbf{0.230} & 0.6328 & 0.9631 & 0.9056 & \textbf{0.9334} & \textbf{1.0000} & 0.4241 & \textbf{0.1881} & 0.2628 & 0.3267 & 0.3583 \\
 & Mycelium 10k & 0.00345 & 5.99 & 1.107 & 0.234 & 0.5713 & \textbf{0.9960} & 0.8137 & 0.8957 & \textbf{1.0000} & \textbf{0.4840} & 0.1710 & 0.2369 & 0.2999 & 0.3314 \\
\midrule
\multirow{3}{*}{Benchmark} & No communication & 0.01866 & \textbf{77.46} & 1.639 & 5.955 & \textbf{0.4256} & 0.7945 & \textbf{0.9968} & 0.8842 & \textbf{1.0000} & 0.7143 & 0.1878 & 0.2950 & 0.3492 & 0.4034 \\
 & Random communication & 0.01805 & 78.24 & \textbf{1.543} & 5.773 & 0.4231 & 0.8721 & 0.9910 & \textbf{0.9278} & 0.9993 & 0.6468 & \textbf{0.1971} & 0.2857 & 0.3135 & 0.3810 \\
 & Mycelium 10k & \textbf{0.01267} & 81.28 & 1.635 & \textbf{5.425} & 0.4001 & \textbf{0.9553} & 0.7198 & 0.8210 & \textbf{1.0000} & \textbf{0.7209} & 0.1958 & \textbf{0.3175} & \textbf{0.3558} & \textbf{0.4114} \\
\midrule
\multirow{2}{*}{\shortstack[l]{(b) Test}} & No angle reference & \textbf{0.00333} & \textbf{5.02} & \underline{16.461} & 0.239 & \textbf{0.6272} & 0.9914 & \textbf{0.8939} & \textbf{0.9401} & \textbf{1.0000} & \textbf{0.4970} & \textbf{0.1774} & \textbf{0.2484} & \textbf{0.3190} & \textbf{0.3497} \\
 & Mycelium 10k & 0.00345 & 5.99 & \textbf{1.107} & \textbf{0.234} & 0.5713 & \textbf{0.9960} & 0.8137 & 0.8957 & \textbf{1.0000} & 0.4840 & 0.1710 & 0.2369 & 0.2999 & 0.3314 \\
\midrule
\multirow{2}{*}{Benchmark} & No angle reference & \textbf{0.01054} & \textbf{73.97} & \underline{24.896} & \textbf{5.265} & \textbf{0.4203} & 0.8193 & \textbf{0.9738} & \textbf{0.8899} & \textbf{1.0000} & 0.7011 & 0.1786 & 0.2804 & 0.3360 & 0.3968 \\
 & Mycelium 10k & 0.01267 & 81.28 & \textbf{1.635} & 5.425 & 0.4001 & \textbf{0.9553} & 0.7198 & 0.8210 & \textbf{1.0000} & \textbf{0.7209} & \textbf{0.1958} & \textbf{0.3175} & \textbf{0.3558} & \textbf{0.4114} \\
\bottomrule
\end{tabular}%
}
\end{table*}

\paragraph{Communication graph.}
Table~\ref{tab:ablations} distinguishes the effects of shortcut connectivity and placement. Removing shortcuts reduces test switch F1 from 0.8957 to 0.3253, but random rewiring outperforms the proposed placement on switch F1 on both splits: 0.9334 versus 0.8957 on test and 0.9278 versus 0.8210 on the benchmark. The proposed placement achieves lower benchmark per-unit voltage error and higher fault type accuracy than random rewiring. These results support the utility of added connectivity for test-set switch inference, but do not establish a general advantage of the proposed placement.

\paragraph{Nominal angle reference.}
Removing the nominal angle reference while keeping the architecture and training protocol fixed increases angle error from 1.107 to 16.461 degrees on the test split and from 1.635 to 24.896 degrees on the frozen benchmark (Table~\ref{tab:ablations}). Thus, the reference reduces angle error approximately fifteenfold on both splits.
% The reference thus improves angle estimation approximately fifteenfold on both splits, although its inclusion introduces measurable tradeoffs in voltage magnitude estimation, phase attribution, and switch state inference.

\subsection{Data Scaling and Grid Size Generalization}

% We study synthetic-network augmentation and the maximum network size seen during training. Both analyses use the 10{,}000-bus model trained with the full synthetic corpus as a common reference.

\begin{table*}[t]
\centering
\caption{Data scaling on the test split and benchmark:
(a) synthetic corpus fractions added to the full reference set with a 10{,}000-bus cap, and (b) training network size caps with full synthetic augmentation. Mycelium 10k is the shared reference, and evaluation sets remain fixed. Best values are bold.}
\label{tab:data_scaling}
\small
\setlength{\tabcolsep}{3pt}
\renewcommand{\arraystretch}{1.1}
\resizebox{\textwidth}{!}{%
\begin{tabular}{cl*{14}{r}}
\toprule
& & \multicolumn{4}{c}{State estimation}
& Phase & \multicolumn{3}{c}{Switch} & \multicolumn{6}{c}{Fault} \\
\cmidrule(lr){3-6}\cmidrule(lr){7-7}
\cmidrule(lr){8-10}\cmidrule(lr){11-16}
Split & Configuration
& p.u. $\downarrow$ & $v_{\mathrm{mag}}$ (V) $\downarrow$
& $v_{\mathrm{ang}}$ (deg) $\downarrow$ & MAPE (\%) $\downarrow$
& Acc. $\uparrow$ & Prec. $\uparrow$ & Rec. $\uparrow$ & F1 $\uparrow$
& Det. F1 $\uparrow$ & class$_5$ $\uparrow$ & Loc. $\uparrow$
& Hop@1 $\uparrow$ & Hop@2 $\uparrow$ & Hop@3 $\uparrow$ \\
\midrule
\multirow{3}{*}{\shortstack[l]{(a) Test}} & Reference only
& 0.00881 & 7.63 & \textbf{0.933} & 0.320
& 0.5846 & \textbf{0.9992} & 0.9178 & 0.9568
& \textbf{1.0000} & 0.3455 & 0.1526 & 0.2305 & 0.2899 & 0.3249 \\
 & +25\% synthetic
& \textbf{0.00302} & \textbf{4.45} & 1.116 & \textbf{0.197}
& \textbf{0.5922} & 0.9806 & \textbf{0.9542} & \textbf{0.9672}
& \textbf{1.0000} & 0.4708 & \textbf{0.1839} & \textbf{0.2820} & \textbf{0.3437} & \textbf{0.3721} \\
 & Mycelium 10k
& 0.00345 & 5.99 & 1.107 & 0.234
& 0.5713 & 0.9960 & 0.8137 & 0.8957
& \textbf{1.0000} & \textbf{0.4840} & 0.1710 & 0.2369 & 0.2999 & 0.3314 \\
\midrule
\multirow{3}{*}{Benchmark} & Reference only
& 0.04025 & 148.88 & 2.388 & 7.714
& 0.3539 & 0.9445 & 0.9946 & 0.9689
& \textbf{1.0000} & 0.2937 & 0.1323 & 0.1799 & 0.2249 & 0.3016 \\
 & +25\% synthetic
& 0.01991 & 141.51 & \textbf{1.134} & 5.669
& 0.3886 & \textbf{0.9585} & \textbf{0.9958} & \textbf{0.9768}
& \textbf{1.0000} & 0.6111 & 0.1429 & 0.2156 & 0.2500 & 0.3056 \\
 & Mycelium 10k
& \textbf{0.01267} & \textbf{81.28} & 1.635 & \textbf{5.425}
& \textbf{0.4001} & 0.9553 & 0.7198 & 0.8210
& \textbf{1.0000} & \textbf{0.7209} & \textbf{0.1958} & \textbf{0.3175} & \textbf{0.3558} & \textbf{0.4114} \\
\midrule
\multirow{3}{*}{\shortstack[l]{(b) Test}} & 5{,}000 buses
& \textbf{0.00322} & \textbf{5.21} & 1.198 & \textbf{0.224}
& 0.5776 & 0.9716 & \textbf{0.9042} & \textbf{0.9367}
& \textbf{1.0000} & 0.4479 & 0.1676 & 0.2501 & 0.3034 & 0.3315 \\
 & 20{,}000 buses
& 0.00393 & 6.96 & \textbf{0.979} & 0.289
& \textbf{0.6150} & 0.9800 & 0.8848 & 0.9299
& 0.9996 & \textbf{0.5837} & \textbf{0.1918} & \textbf{0.2793} & \textbf{0.3524} & \textbf{0.3947} \\
 & Mycelium 10k
& 0.00345 & 5.99 & 1.107 & 0.234
& 0.5713 & \textbf{0.9960} & 0.8137 & 0.8957
& \textbf{1.0000} & 0.4840 & 0.1710 & 0.2369 & 0.2999 & 0.3314 \\
\midrule
\multirow{3}{*}{Benchmark} & 5{,}000 buses
& \textbf{0.00963} & \textbf{73.11} & \textbf{1.071} & \textbf{5.177}
& 0.4234 & 0.9503 & \textbf{0.9934} & \textbf{0.9713}
& \textbf{1.0000} & 0.6852 & 0.1799 & 0.3016 & 0.3267 & 0.3955 \\
 & 20{,}000 buses
& 0.01063 & 75.72 & 1.456 & 5.332
& \textbf{0.4387} & 0.9258 & 0.5574 & 0.6959
& 0.9980 & 0.7063 & \textbf{0.1971} & 0.2897 & 0.3399 & 0.4008 \\
 & Mycelium 10k
& 0.01267 & 81.28 & 1.635 & 5.425
& 0.4001 & \textbf{0.9553} & 0.7198 & 0.8210
& \textbf{1.0000} & \textbf{0.7209} & 0.1958 & \textbf{0.3175} & \textbf{0.3558} & \textbf{0.4114} \\
\bottomrule
\end{tabular}%
}
\end{table*}

\paragraph{Synthetic network scaling.}
All runs include the full reference training set. We add either no synthetic networks, a nested 25\% subset matched to the full corpus's network size distribution, or the full synthetic corpus. Architecture, optimization, and training budget are held fixed, and all runs are evaluated on the same unseen reference networks.
% Table~\ref{tab:data_scaling}(a) shows substantial cross-family gains from synthetic augmentation. Relative to reference-only training, the full synthetic corpus reduces benchmark per-unit voltage error by 69\%, raises fault-type accuracy from 0.2937 to 0.7209, and improves phase accuracy from 0.3539 to 0.4001. Notably, adding just 25\% of the synthetic networks nearly halves voltage error on the benchmark and yields the best test-set voltage error, switch F1, and exact fault-localization accuracy. The improvement from 25\% synthetic augmentation suggests that additional network diversity helps the model learn electrical relationships that transfer to unseen reference networks, even when the generated networks differ from the test distribution.
{Synthetic augmentation improves several cross-family metrics, with nonmonotonic gains across tasks.}
Table~\ref{tab:data_scaling}(a) shows that the full synthetic corpus reduces benchmark per-unit voltage error by 69\% relative to reference-only training and raises fault type accuracy from 0.2937 to 0.7209. The 25\% subset instead achieves the best test-set voltage magnitude errors, switch F1, and exact localization. On the benchmark, switch F1 falls from 0.9768 at 25\% to 0.8210 with the full corpus. Additional synthetic coverage therefore benefits several tasks without uniformly improving performance.

\paragraph{Maximum training network size.}
We train models with network size caps of 5{,}000, 10{,}000, and 20{,}000 buses, keeping the evaluation sets unchanged. The 20{,}000-bus run requires a batch size of four rather than eight, so this comparison varies effective batch size. Larger training networks do not uniformly improve generalization. Table~\ref{tab:data_scaling}(b) shows that the 20{,}000-bus run achieves the highest test-set phase accuracy, fault type accuracy, and exact fault localization accuracy, but has the highest test-set voltage magnitude error. On the frozen benchmark, the 5{,}000-bus run has the lowest state estimation errors and highest switch F1, whereas the 20{,}000-bus run has the highest phase accuracy. Exposure to larger networks therefore does not improve every task, and the batch size difference prevents attributing these changes solely to network size.
\subsection{Multi-task vs. Single-task}
\label{sec:single-task}
\begin{table*}[t]
\centering
\caption{Joint versus single-task training with the same backbone on the test split and frozen benchmark. Dashes mark untrained tasks. Lower is better for state estimation errors and higher otherwise. Best values per split are bold.}
\label{tab:single_vs_multitask}
\resizebox{\textwidth}{!}{%
\begin{tabular}{llcccccccccccccc}
\toprule
& \textbf{Run} &
\multicolumn{4}{c}{\textbf{State estimation}} & \textbf{Phase} &
\multicolumn{3}{c}{\textbf{Switch}} & \multicolumn{6}{c}{\textbf{Fault}} \\
\cmidrule(lr){3-6}\cmidrule(lr){7-7}\cmidrule(lr){8-10}\cmidrule(lr){11-16}
Split & &
p.u. $\downarrow$ &
$v_{\mathrm{mag}}$ (V) $\downarrow$ &
$v_{\mathrm{ang}}$ (deg) $\downarrow$ &
$v_{\mathrm{mag}}$ MAPE (\%) $\downarrow$ &
acc $\uparrow$ &
prec $\uparrow$ & rec $\uparrow$ & F1 $\uparrow$ &
det F1 $\uparrow$ & class$_5$ $\uparrow$ & exact $\uparrow$ &
hop@1 $\uparrow$ & hop@2 $\uparrow$ & hop@3 $\uparrow$ \\
\midrule
\multirow{5}{*}{Test}
& SE only &
  \textbf{0.00148} & \textbf{3.33} & \textbf{1.057} & \textbf{0.148} &
  - & - & - & - & - & - & - & - & - & - \\
& Phase only &
  - & - & - & - &
  0.5363 &
  - & - & - & - & - & - & - & - & - \\
& Switch only &
  - & - & - & - & - &
  0.8325 & \textbf{0.9374} & 0.8818 &
  - & - & - & - & - & - \\
& Fault only &
  - & - & - & - & - & - & - & - &
  0.7990 & \textbf{0.5322} & \textbf{0.1927} & \textbf{0.2731} & \textbf{0.3441} & \textbf{0.3768} \\
& Mycelium 10k &
  0.00345 & 5.99 & 1.107 & 0.234 &
  \textbf{0.5713} &
  \textbf{0.9960} & 0.8137 & \textbf{0.8957} &
  \textbf{1.0000} & 0.4840 & 0.1710 & 0.2369 & 0.2999 & 0.3314 \\
\midrule
\multirow{5}{*}{Benchmark}
& SE only &
  0.07072 & \textbf{21.51} & 7.038 & 12.067 &
  - & - & - & - & - & - & - & - & - & - \\
& Phase only &
  - & - & - & - &
  0.3683 &
  - & - & - & - & - & - & - & - & - \\
& Switch only &
  - & - & - & - & - &
  0.7870 & \textbf{0.9961} & \textbf{0.8793} &
  - & - & - & - & - & - \\
& Fault only &
  - & - & - & - & - & - & - & - &
  0.8405 & 0.6746 & \textbf{0.2169} & \textbf{0.3188} & 0.3320 & 0.3889 \\
& Mycelium 10k &
  \textbf{0.01267} & 81.28 & \textbf{1.635} & \textbf{5.425} &
  \textbf{0.4001} &
  \textbf{0.9553} & 0.7198 & 0.8210 &
  \textbf{1.0000} & \textbf{0.7209} & 0.1958 & 0.3175 & \textbf{0.3558} & \textbf{0.4114} \\
\bottomrule
\end{tabular}%
}
\end{table*}
We compare the cross-task Mycelium model with four models that use the same backbone but receive supervision for only one task. Each single-task model is evaluated on its trained task. Joint training provides selective cross-family benefits.
Table~\ref{tab:single_vs_multitask} reveals a reversal in state estimation performance across splits. The single-task model has lower test per-unit error than joint training, 0.00148 versus 0.00345, but substantially higher benchmark error, 0.07072 versus 0.01267. Joint training also improves phase accuracy on both splits and benchmark fault type accuracy. Its benefits remain selective: single-task training retains lower benchmark voltage error in volts, higher benchmark switch F1, and better exact fault localization on both splits. Shared supervision therefore improves some aspects of cross-family generalization while leaving advantages for specialization.

\subsection{Single-task finetuning}
\label{sec:finetune_scope}
We finetune Mycelium 10k using only the target task loss. Head-only runs freeze the encoders and backbone and use reference networks. Full-model runs update all parameters using reference and synthetic networks, except for switch inference, which uses reference networks in both settings. For the other tasks, parameter scope and data composition therefore vary together.
\begin{figure}[t]
\centering
\includegraphics[width=\linewidth]{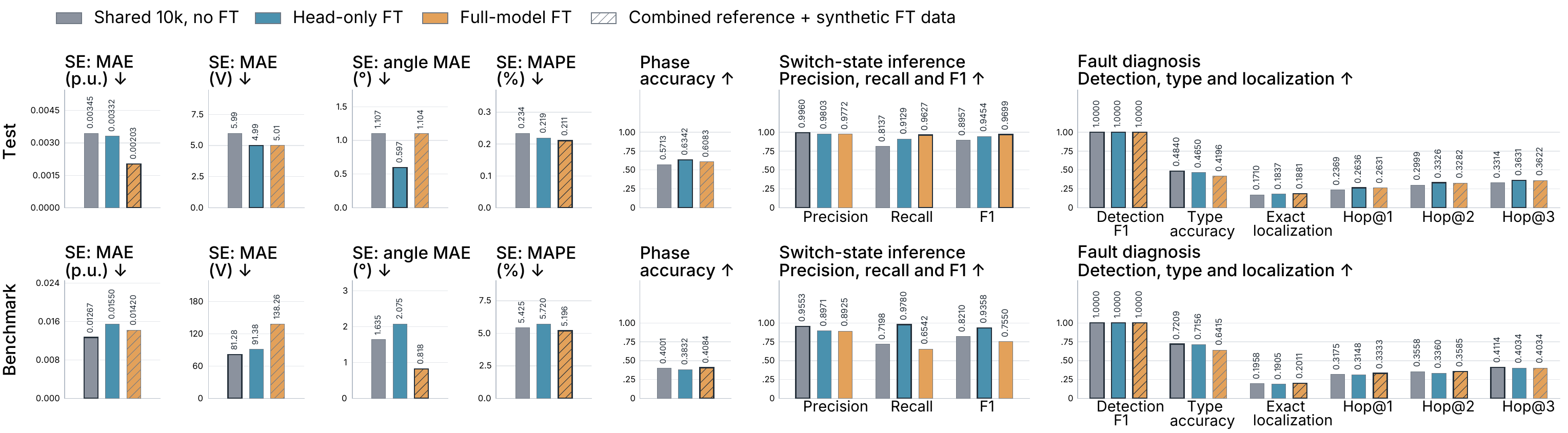}
\caption{Single-task finetuning on the test split and frozen benchmark, shown in the top and bottom rows. Hatching denotes combined reference and synthetic finetuning data. Lower is better for state estimation errors and higher is better otherwise.}
\label{fig:finetune_all_tasks}
\end{figure}
Figure~\ref{fig:finetune_all_tasks} shows that finetuning improves test-set state estimation, phase accuracy, switch F1, and fault localization but reduces fault type accuracy. These gains do not consistently transfer to the benchmark. Head-only switch adaptation raises benchmark F1 from 0.8210 to 0.9358, whereas full-model adaptation reaches a higher test F1 (0.9699) but lowers benchmark F1 to 0.7550. This suggests the pretrained backbone encodes a powerful transferable representation that a frozen feature head preserves and full finetuning overwrites. Yet head-only phase adaptation also trades benchmark accuracy for test accuracy, so task specialization can improve in-family performance without improving cross-family generalization, even with a frozen backbone.

\Needspace{10\baselineskip}
\begin{samepage}
\section{Conclusion}

We introduced Mycelium, a shared model for four inference tasks across distribution grids. Synthetic augmentation and joint supervision improve selected cross-family metrics, while nominal angle references substantially reduce angle error. Communication edge placement and task-specific finetuning introduce tradeoffs, and gains on familiar source families do not consistently extend to excluded families. These results identify both benefits and limits of shared learning across grids.

\end{samepage}

\bibliography{iclr2027_conference}
\bibliographystyle{iclr2027_conference}

\appendix
% \section{Appendix}

\section{Data}

\subsection{Reference Data Sources}
\label{app:reference_data_sources}
The reference arm of the corpus comprises publicly available distribution network models from four source families:
(a) \textbf{NREL SMART-DS:} networks covering the Austin, San Francisco, and Greensboro regions, representing realistic synthetic U.S.\ distribution systems with multiple medium voltage classes and detailed secondary networks \citep{OEDI_Dataset_2981}.
(b) \textbf{Standard benchmark networks:} EPRI circuits \footnote{EPRI OpenDSS Test Circuits: \url{https://sourceforge.net/p/electricdss/code/HEAD/tree/trunk/Distrib/EPRITestCircuits/}} (\texttt{ckt5}, \texttt{ckt7}, and \texttt{ckt24}),
the IEEE test feeder suite, including the IEEE European low voltage
network \citep{schneider2017analytic} and D-Suite networks \citep{Deakin2025}.
(c) \textbf{Iowa State test systems:} 240- and 296-bus real distribution systems released as OpenDSS models with one year of smart meter data \citep{bu2019time}.\footnote{Iowa Distribution
Test Systems: \url{https://wzy.ece.iastate.edu/testsystem.html}}
(d) \textbf{Nonsynthetic European systems:} utility derived networks based on GIS and smart meter data, including a hybrid MV/LV industrial network and large rural and urban LV networks, the largest containing 6,738 buses \citep{taye2024set}.

\subsection{Structural Diversity of Reference Networks}
\label{app:network_characteristics}
Figure~\ref{fig:network_characteristics} characterizes the reference corpus by bus count, independent loop count, electrical extent, and customer density. The benchmark networks occupy distinct parts of these distributions: the Iowa networks have relatively small bus counts and electrical extents, whereas the European networks lie at the 96th-100th percentiles in independent loop count. The urban network is near the median in bus count but at the 100th percentile in customer density. These contrasts highlight structural differences that network size alone does not capture.

\begin{figure*}[t]
\centering
\includegraphics[width=\textwidth]{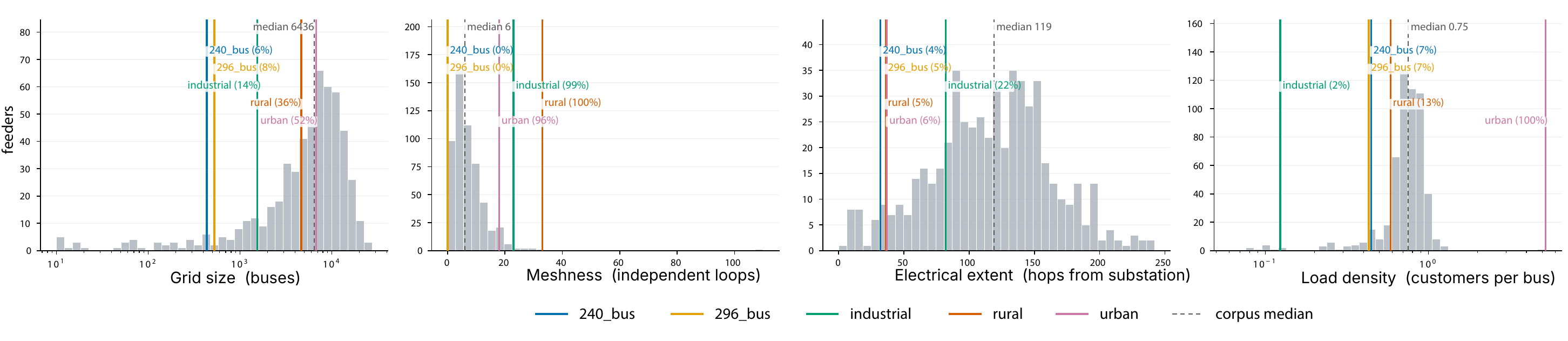}
\caption{Structural characteristics of the reference corpus and the five frozen benchmark networks: bus count (a), independent loop count (b), electrical extent in hops from the substation (c), and customer density (d). Histograms show corpus distributions, colored vertical lines mark benchmark networks, and dashed lines indicate corpus medians. The table reports benchmark percentile ranks. Bus count and customer density use logarithmic horizontal axes.}
\label{fig:network_characteristics}
\end{figure*}

\subsection{Synthetic Network Generation}
\label{app:synthetic_generation}

Algorithm~\ref{alg:synthetic_generation} summarizes the procedure described below. We define $\mathcal D_{\mathrm{ref}}$ as the training subset of the reference network corpus. The component library, customer load records, and source specific or pooled planning statistics used by the generator are derived exclusively from this subset. No frozen benchmark network contributes to synthetic network generation.

We construct synthetic distribution networks and operating profiles by sampling geographic layouts, equipment conventions, and customer demand, then checking the resulting electrical behavior in OpenDSS~\citep{dugan2011open}. The procedure combines geographic feeder synthesis~\citep{saha2019framework} with empirical component sampling and simulation guided repair. Component parameters, customer load records, and planning statistics are obtained from a reference corpus $\mathcal{D}_{\text{ref}}$. Operating profiles are drawn from the datasets described below. Each accepted case contains an electrical network and one 24-hour operating window.

\paragraph{Geography and component sampling.}
Street segments and building footprints from Overture Maps~\citep{overture2026} provide candidate cable routes and customer locations. We sample a territory, choose a feeder head location, and snap it to the nearest street node. We keep the streets reachable from that node and select nearby buildings to serve. A component library built from $\mathcal{D}_{\text{ref}}$ supplies conductor impedances and ratings, transformer parameters, customer load records, and control settings. A sampled source dataset convention specifies customers per transformer, transformer utilization, secondary voltage classes, and device prevalence. When dataset specific statistics are unavailable, we use pooled values from $\mathcal{D}_{\text{ref}}$. Load records are sampled within single- and three-phase classes, with source weights tempered to reduce domination by larger datasets. Groups identified as placeholder or aggregate spot loads are excluded. Geography and the source dataset convention are sampled separately.

\paragraph{Network construction.}
We cluster customers spatially into service areas, each supplied by one distribution transformer. The initial cluster count is determined by the sampled number of customers per transformer from $\mathcal{D}_{\text{ref}}$. We increase this count when capacity, reach, or estimated voltage drop checks fail. Transformers are placed near each cluster's demand weighted centroid. Eligible clusters containing only single-phase customers receive center-tapped transformers with individual service drops. The remaining clusters use three-phase units and shortest path street routing. Approximate Steiner trees~\citep{mehlhorn1988} connect transformer and substation locations at higher voltage levels. Downstream per phase demand determines conductor and transformer sizing. The construction also samples sectionalizing switches, normally open ties, capacitors, voltage controls, and rooftop PV. These choices vary network size, branching, phase composition, and equipment. At this stage, each customer has a nominal demand. Time series are assigned next.

\paragraph{Operating profiles and load coincidence.}
Nominal active and reactive demands are sampled from the component library and rescaled to a sampled network demand level. We assign residential ResStock profiles to single-phase loads and commercial ComStock profiles to three-phase loads~\citep{wilson2022}. Profiles are drawn from a sampled climate group, with annual peak ranks matched to nominal load sizes. Let $\bar p_i$ be customer $i$'s fixed nominal active demand and $a_i(h)$ its assigned reference building demand profile in kW, where $h\in\{1,\ldots,8760\}$ indexes the hours of the year. We define the synthetic customer's nominal demand schedule as
\begin{equation}
    p_i(h)=\bar p_i\frac{a_i(h)}{\max_{h'}a_i(h')},
    \qquad
    \eta_c=\frac{\max_h\sum_{i\in c}p_i(h)}
                 {\sum_{i\in c}\bar p_i},
    \label{eq:synthetic_coincidence}
\end{equation}
where $c$ is the set of customers served by one transformer. The reference profile supplies the temporal pattern, while $\bar p_i$ sets its annual peak. The coincidence factor $\eta_c$ accounts for customers reaching their individual peaks at different times. A common date from a sampled season defines the 24-hour window. Reactive power and PV output profiles are drawn from SMART-DS~\citep{OEDI_Dataset_2981}. Reactive profiles are normalized by their annual peaks and scaled by nominal reactive demand. PV profiles combine a common solar trajectory with site level variation. Geography, load climate, and solar region are sampled separately rather than jointly matched.

\paragraph{Snapshot simulation.}
We first construct a static OpenDSS model from the sampled network and nominal customer demands. At this stage, power flow is solved in \emph{snapshot mode}: one operating point is evaluated at a time. Annual load profiles are used to compute the coincidence factors $\eta_c$, but are not yet applied as a time stepped simulation.

An operating input $\mathbf{U}=(\mathbf p,\mathbf q,\boldsymbol\alpha)$ specifies the nominal active and reactive demands of all customers and the availability of DER generation. A DER multiplier $\alpha_j=0$ disables generation at device $j$, while $\alpha_j=1$ specifies full availability under its configured device model. With $c(i)$ denoting the transformer service area containing customer $i$, the two stress inputs are
\begin{equation}
\begin{aligned}
\mathbf{U}_{\mathrm A}
&=\left((\eta_{c(i)}\bar p_i)_i,\;(\eta_{c(i)}\bar q_i)_i,\;\mathbf 0\right),\\
\mathbf{U}_{\mathrm B}
&=\left(\lambda\bar{\mathbf p},\;\lambda\bar{\mathbf q},\;\mathbf 1\right),
\qquad \lambda\sim\mathrm{Uniform}(0.2,0.4),
\end{aligned}
\label{eq:synthetic_snapshot_inputs}
\end{equation}
where $\bar q_i$ is nominal reactive demand. Thus, $\mathbf{U}_{\mathrm A}$ represents coincident peak demand with DER generation disabled, while $\mathbf{U}_{\mathrm B}$ represents reduced demand with full DER availability. The multiplier $\lambda$ is drawn once per proposed network and retained during repair.

The call $Y=\operatorname{SnapshotPF}(\mathcal G,\mathbf c;\mathbf{U})$ applies one such input to the network and runs an unbalanced OpenDSS snapshot solve. Here $\mathbf c$ contains equipment parameters and configurations. The returned diagnostics $Y$ include convergence status, bus voltages, line currents, and line and transformer loading ratios. We solve the two operating points separately to obtain $Y_{\mathrm A}$ and $Y_{\mathrm B}$.

\paragraph{Repairs to the proposed network.}
Repairs are offline edits to the synthetic simulator model during data generation. The editable variables are customer phase assignments, transformer taps and ratings, conductor multiplicities, capacitor and regulator placement and ratings, and DER inverter modes and output limits. Let $\mathcal R_{\mathrm C}$ denote the ordered snapshot repair schedule and $\mathcal R_{\mathrm D}$ the daily repair schedule in Table~\ref{tab:synthetic_repairs}. Each scheduled operation $r$ has a trigger based on the current diagnostics and a repetition limit. When triggered, it edits the affected variables in $(\mathcal G,\mathbf c)$. For example, a line overload can trigger an additional parallel conductor, whereas overvoltage can trigger a tap change or DER curtailment. Repairs for overloading are enabled only when the thermal enforcement flag $b$ is set. After every edit, the relevant operating points are solved again: improving one can worsen another.

\begin{table}[!htbp]
\centering
\small
\setlength{\tabcolsep}{4pt}
\renewcommand{\arraystretch}{1.15}
\caption{Repair actions used during synthetic network generation. Actions apply to the affected service areas or branches, subject to the sampled design's equipment options. Snapshot repairs also seek voltage headroom beyond the final acceptance limits.}
\label{tab:synthetic_repairs}
\begin{tabularx}{\linewidth}{@{}>{\raggedright\arraybackslash}p{0.18\linewidth}>{\raggedright\arraybackslash}X>{\raggedright\arraybackslash}p{0.22\linewidth}@{}}
\toprule
Trigger & Change to the proposed model & Repetition limit \\
\midrule
Low snapshot voltage &
Rebalance single-phase loads where connections permit; adjust transformer taps; add capacitors or regulators; add parallel conductors along affected paths; select a larger transformer. &
Ordered passes; up to 4 conductor reinforcement and 2 transformer upgrade rounds. \\
High snapshot voltage &
Enable inverter volt-var control, back off voltage boosting taps where the low voltage condition permits, reinforce conductors, or reduce affected DER output limits by 20\%. &
Up to 5 rounds. \\
Snapshot overload, with thermal enforcement &
Add parallel conductor circuits, choose a larger transformer, add parallel substation units, or increase regulator capacity. &
Up to 3 rounds. \\
Daily voltage violation &
Adjust taps or reinforce affected paths for undervoltage; use inverter control, tap backoff, phase rebalancing, or DER curtailment for overvoltage. &
Up to 2 undervoltage and 8 overvoltage rounds. \\
\bottomrule
\end{tabularx}
\end{table}

\paragraph{Daily simulation and acceptance.}
After snapshot repair, we attach the selected 24-hour load and DER generation profiles. The call $Y_{1:24}=\operatorname{DailyPF}(\mathcal G,\mathbf c;\mathbf{U}_{1:24})$ switches OpenDSS to daily mode and performs 24 sequential one-hour solves, returning diagnostics at each step. The daily repairs in Table~\ref{tab:synthetic_repairs} address violations identified along this trajectory. After each change, both snapshot conditions and the complete day are checked again. The load profiles and sampled multipliers remain fixed during these repairs.

Acceptance requires power flow convergence, a minimum voltage of at least $0.95$~p.u. in snapshot A, and voltages within $[0.95,1.05]$~p.u. in snapshot B and throughout the sampled day. Daily voltage checks omit values at or below $0.05$~p.u. to exclude floating tie nodes. An entirely collapsed step is rejected. For each proposed network, thermal enforcement is enabled with probability $0.7$ and remains fixed during repair. When enabled, line and transformer ratings are checked in both snapshots, and line current ratings are checked at every timestep of the day. Cases that still fail the applicable checks after the allowed repairs are rejected.

\begin{algorithm}[!htb]
\small
\DontPrintSemicolon
\SetAlgoNlRelativeSize{-1}
\SetCommentSty{emph}
\caption{Synthetic network generation}
\label{alg:synthetic_generation}
\KwIn{Reference corpus $\mathcal D_{\mathrm{ref}}$; street and building maps; load and DER generation profiles $\mathcal P$; seeds $\mathcal S$; repair schedules $\mathcal R_{\mathrm C},\mathcal R_{\mathrm D}$ defined in the text and Table~\ref{tab:synthetic_repairs}.}
\KwOut{Accepted networks with 24 validated hourly operating points.}
Generate component library $\mathcal L$ and occurrence statistics from $\mathcal D_{\mathrm{ref}}$ using OpenDSS\;
\tcp{An unrecovered error or timeout records failure and skips the current seed.}
\ForEach{$s\in\mathcal S$}{
    Initialize the random generator with seed $s$\;
    Sample territory, feeder head location, and equipment nameplate parameters\;
    Draw $\lambda\sim\mathrm{Uniform}(0.2,0.4)$ and $b\sim\mathrm{Bernoulli}(0.7)$\;
    \tcp{Stage 1: construct and repair a snapshot model.}
    Keep streets reachable from the feeder head and select nearby customer locations\;
    Assign customer phase classes, then sample and rescale nominal demands $(\bar p_i,\bar q_i)$ from $\mathcal L$\;
    Sample DER presence, capacities, and inverter settings from $\mathcal D_{\mathrm{ref}}$\;
    Cluster customers, site transformers, and route service and upstream connections\;
    Assign phases and size equipment\;
    Add sectionalizing switches and normally open ties to obtain $(\mathcal G,\mathbf c)$\;
    Build the OpenDSS model in snapshot mode with nominal demands $(\bar{\mathbf p},\bar{\mathbf q})$\;
    Use annual profiles from $\mathcal P$ to compute $\eta_c$ in Eq.~\eqref{eq:synthetic_coincidence}\;
    Form the snapshot inputs $\mathbf{U}_{\mathrm A},\mathbf{U}_{\mathrm B}$ in Eq.~\eqref{eq:synthetic_snapshot_inputs}\;
    $Y_{\mathrm A}\leftarrow\operatorname{SnapshotPF}(\mathcal G,\mathbf c;\mathbf{U}_{\mathrm A})$;
    $Y_{\mathrm B}\leftarrow\operatorname{SnapshotPF}(\mathcal G,\mathbf c;\mathbf{U}_{\mathrm B})$\;
    \ForEach{triggered operation $r$ in $\mathcal R_{\mathrm C}$, within its repetition limit}{
        $(\mathcal G,\mathbf c)\leftarrow r(\mathcal G,\mathbf c;Y_{\mathrm A},Y_{\mathrm B},b)$\;
        Solve both snapshots again to update $Y_{\mathrm A},Y_{\mathrm B}$\;
    }
    \tcp{Stage 2: attach time series and validate one day.}
    Select a common date and construct load and DER generation inputs $\mathbf{U}_{1:24}$\;
    $Y_{1:24}\leftarrow\operatorname{DailyPF}(\mathcal G,\mathbf c;\mathbf{U}_{1:24})$\;
    \ForEach{triggered operation $r$ in $\mathcal R_{\mathrm D}$, within its repetition limit}{
        $(\mathcal G,\mathbf c)\leftarrow r(\mathcal G,\mathbf c;Y_{\mathrm A},Y_{\mathrm B},Y_{1:24},b)$\;
        Solve both snapshots and all 24 hourly steps again to update $Y_{\mathrm A},Y_{\mathrm B},Y_{1:24}$\;
    }
    \eIf{$\operatorname{Accept}(Y_{\mathrm A},Y_{\mathrm B},Y_{1:24};b)$}{
        Save $(\mathcal G,\mathbf c,\mathbf{U}_{1:24})$ as an accepted case\;
    }{
        Record rejection\;
    }
    Record seed, sampled parameters, repair history, and final diagnostics\;
}
\end{algorithm}

The predicate $\operatorname{Accept}$ applies the convergence, voltage, and conditional rating checks described above to the final snapshot and daily diagnostics. The flag $b$ controls thermal enforcement. The variables $\lambda$ and $b$ remain fixed throughout repair.

\paragraph{Outputs and scope.}
The generator saves self contained OpenDSS models, load and solar time series, sampled parameters, repair histories, and acceptance results. Observation noise and missingness are specified separately by the sensor policies in Appendix~\ref{app:sensor_policies}. Repair and rejection change the distribution of accepted cases relative to the sampling priors.

\section{Task Overview}
\label{app:task_overview}

We jointly infer electrical states, connectivity attributes, and fault
events from a network's equipment graph and partial, noisy measurements.
All four tasks share the same inputs. Labels and target validity masks
are used only for supervision and evaluation.

\subsection{Network Graph and Partial Observations}
\label{sec:task_graph}

We represent each network as a typed graph
\begin{equation}
    \mathcal{G}
    =
    (\mathcal{V},\mathcal{E},
    \tau_{\mathcal{V}},\tau_{\mathcal{E}}),
\end{equation}
with bus, consumer, substation, capacitor, and DER nodes. Relations
distinguish lines, transformers, reactors, and device to bus attachments.
All switchable branches remain in the graph regardless of their
operating state: the equipment inventory is known, while switch states
must be inferred.

Each sample contains $T=24$ consecutive hourly observations, indexed
by $\mathcal{T}=\{0,\ldots,T-1\}$. Entity
$u\in\mathcal{V}\cup\mathcal{E}$ has static attributes $\mathbf{s}_u$,
temporal measurements
$\mathbf{X}_u\in\mathbb{R}^{T\times d_{\tau(u)}}$, and an observation
mask $\mathbf{M}_u\in\{0,1\}^{T\times d_{\tau(u)}}$.
Static attributes describe equipment properties. Measurements include
available voltage, power, current, and feeder head telemetry.
Missing readings are zero filled and masked to distinguish them from
physical zeros.

The model receives
\begin{equation}
    \mathcal{O}
    =
    \left\{
        (\mathbf{M}_u\odot\mathbf{X}_u,\mathbf{M}_u,\mathbf{s}_u)
        : u\in\mathcal{V}\cup\mathcal{E}
    \right\},
\end{equation}
and jointly predicts
\begin{equation}
    f_{\boldsymbol{\theta}}(\mathcal{G},\mathcal{O})
    =
    \big(
        \widehat{\mathbf{V}},
        \widehat{\boldsymbol{\Theta}},
        \widehat{\mathbf{P}}^{\mathrm{phase}},
        \widehat{\mathbf{p}}^{\mathrm{sw}},
        \widehat{\mathbf{p}}^{\mathrm{fault}},
        \widehat{\mathbf{p}}^{\mathrm{loc}}
    \big).
    \label{eq:unified_prediction}
\end{equation}
Sensor policies vary measurement placement, density, noise, and
missingness without changing the task definitions.

\subsection{State Estimation}
\label{sec:task_state_estimation}

For each bus, timestep, and phase in
$\mathcal{P}=\{\mathrm{A},\mathrm{B},\mathrm{C}\}$,
the model predicts voltage magnitude and angle:
\begin{equation}
    \widehat{\mathbf{V}},\widehat{\boldsymbol{\Theta}}
    \in
    \mathbb{R}^{|\mathcal{V}_{\mathrm{bus}}|\times T\times 3}.
\end{equation}
A validity mask $\mathbf{M}^{\mathrm{state}}$ of the same shape
excludes absent phases and invalid targets. The task reconstructs
electrical states over the input window, with supervision and
evaluation restricted to valid (bus, timestep, phase) entries.

\subsection{Phase Attribution}
\label{sec:task_phase_attribution}

Phase attribution applies to eligible single-phase consumers whose
connection phase is not determined trivially by the local structure.
For each eligible consumer $j$, the model predicts a distribution
$\widehat{\mathbf{p}}^{\mathrm{phase}}_j\in[0,1]^3$ over the canonical
$\mathrm{A}$-$\mathrm{B}$-$\mathrm{C}$ ordering and assigns the
phase with the highest probability. The physical phase label remains constant
throughout the window.

Separate masks track structural eligibility and the availability
of at least one customer voltage observation, supporting evaluation
over either all eligible consumers or the observed subset.

\subsection{Switch State Inference}
\label{sec:task_switch_state}

For each switchable line $e\in\mathcal{E}_{\mathrm{sw}}$, the target
is $y_e^{\mathrm{sw}}=1$ for open and $0$ for closed. Equipment
metadata identifies switch locations. The model predicts open state
probabilities from network structure and available electrical
measurements, with operating state labels withheld. Switch states
remain constant within each window.

\subsection{Fault Classification and Conditional Localization}
\label{sec:task_fault}

Fault diagnosis combines graph level classification with equipment
localization. The model predicts a distribution over
\begin{equation}
    \mathcal{C}_{\mathrm{fault}}
    =
    \{
        \mathrm{normal},
        \mathrm{LG},
        \mathrm{LL},
        \mathrm{LLG},
        \mathrm{LLL},
        \mathrm{LLLG}
    \},
\end{equation}
where $\mathrm{normal}$ denotes no active fault.

For faulted samples, the location target identifies an element of
the joint candidate set
\begin{equation}
    \mathcal{A}_{\mathrm{fault}}
    =
    \mathcal{V}_{\mathrm{bus}}
    \mathbin{\dot{\cup}}
    \mathcal{E}_{\mathrm{line}}
    \mathbin{\dot{\cup}}
    \mathcal{E}_{\mathrm{xfmr}}.
\end{equation}
The model scores all candidate buses, lines, and transformers together
and predicts
$\widehat{y}^{\mathrm{loc}}
=\arg\max_{a\in\mathcal{A}_{\mathrm{fault}}}
\widehat{p}^{\mathrm{loc}}_a$.
Normal samples have no location target and contribute no localization
loss.

The current benchmark injects one fault at the final observed
timestep, $t=T-1$. Diagnosis therefore uses the pre event history
and initial fault response, without post fault recovery measurements.
Arbitrary onset times and sustained faults are outside this evaluation.

\subsection{Joint Task Applicability}
\label{sec:task_applicability}

A task applicability vector $\mathbf{a}_n\in\{0,1\}^4$ specifies
which tasks have valid targets for sample $n$. Combined with
entity level validity masks, it allows one window to supervise
multiple tasks. These supervision masks remain separate from the
observation masks in $\mathcal{O}$ and are never model inputs.
Every sample uses the same $(\mathcal{G},\mathcal{O})$ interface
without an explicit task identifier.

% Preamble: \usepackage{booktabs,graphicx}

% \begin{table}[t]
% \centering
% \caption{Prediction targets and temporal scope of the four tasks.}
% \label{tab:task_summary}
% \resizebox{\linewidth}{!}{%
% \begin{tabular}{@{}llll@{}}
% \toprule
% Task & Prediction scope & Target & Temporal scope \\
% \midrule
% State estimation
% & Bus-phase
% & Voltage magnitude and angle
% & Per timestep \\
% Phase attribution
% & Eligible consumer
% & $\mathrm{A}/\mathrm{B}/\mathrm{C}$ connection
% & Constant within window \\
% switch state inference
% & Switchable line
% & Open or closed
% & Constant within window \\
% Fault diagnosis
% & Graph and equipment
% & Fault class and conditional location
% & Final-timestep event \\
% \bottomrule
% \end{tabular}%
% }
% \end{table}
\begin{table}[t]
\centering
\small
\setlength{\tabcolsep}{4pt}
\caption{Prediction problems and temporal scope.}
\label{tab:task_summary}
\resizebox{\linewidth}{!}{%
\begin{tabular}{@{}l l l@{}}
\toprule
Task & ML formulation & Prediction scope \\
\midrule
State estimation (SE) & Continuous regression & Magnitude and angle per (bus, phase, timestep) \\
Phase attribution & Three-class classification & One A/B/C assignment per eligible consumer and window \\
Switch state inference & Binary classification & One open/closed state per switch and window \\
Fault diagnosis & Six-class classification; conditional localization & Normal/fault class per window; equipment location for faulted windows \\
\bottomrule
\end{tabular}%
}
\end{table}

\section{Model Details}
\label{app:model_details}

Time indices in this appendix run from $1$ to $T$, corresponding to $0$ to
$T-1$ in Appendix~\ref{app:task_overview}. The final observed timestep is
therefore $T$ here.

\subsection{Input Schema and Supervision Masks}
\label{app:input_schema}

Table~\ref{tab:input_schema} summarizes the node, relation, and auxiliary inputs. Static attributes are shared across a
sample's $T=24$ hourly steps, while temporal measurements retain their
time dimension. Phase resolved measurements follow the fixed
$(A,B,C)$ ordering. Consumer measurements are scalar channels without
an explicit consumer phase label. Capacitor and distributed energy
resource (DER) nodes are included when present in the network.
Features marked as optional are controlled by model configuration.

Switch eligible lines remain in the structural graph regardless of their
open/closed state. Customer assignments to transformers identify serving
equipment. They are distinct from the customer phase labels being predicted.
Transformer winding mappings likewise describe equipment metadata.

\begingroup
\small
\setlength{\tabcolsep}{4pt}
\renewcommand{\arraystretch}{1.05}
\setlength{\LTcapwidth}{\textwidth}
\begin{longtable}{@{}
 >{\raggedright\arraybackslash}p{0.18\textwidth}
 >{\raggedright\arraybackslash}p{0.43\textwidth}
 >{\raggedright\arraybackslash}p{\dimexpr0.39\textwidth-4\tabcolsep\relax}@{}}
\caption{Model inputs for nodes (a), relations (b), and auxiliary stores (c). Every temporal measurement has an observability mask. Connectivity only relations provide endpoint indices and relation identity without separate attribute vectors. DT denotes distribution transformer. Dashes indicate no direct temporal input. Auxiliary entries describe their inputs and role across the last two columns.}\label{tab:input_schema}\\
\toprule
\textbf{Type} & \textbf{Static inputs / connectivity} & \textbf{Temporal inputs} \\
\midrule
\endfirsthead
\caption[]{Model inputs (continued).}\\
\toprule
\textbf{Type} & \textbf{Static inputs / connectivity} & \textbf{Temporal inputs} \\
\midrule
\endhead
\midrule
\multicolumn{3}{r@{}}{Continued on the next page.}\\
\endfoot
\bottomrule
\endlastfoot
\multicolumn{3}{@{}l}{\textbf{(a) Nodes}} \\
\addlinespace[0.25em]
\textbf{Bus}
& Base voltage; source bus indicator; normalized coordinates and coordinate validity flag.
& Per phase voltage magnitude and angle; active power $P$ and reactive power $Q$. \\
\addlinespace[0.25em]
\textbf{Consumer}
& One-hot metering tier encoding.
& Customer voltage channels for phase attribution and voltage derived from the metering tier; active and reactive power. \\
\addlinespace[0.25em]
\textbf{Substation}
& Positive- and zero-sequence source resistance and reactance; three-phase and single-phase short-circuit strengths; base voltage; voltage setpoint; source angle; frequency.
& Feeder head voltage and active/reactive power when feeder head fusion is enabled. \\
\addlinespace[0.25em]
\textbf{Capacitor}
& Voltage and reactive power ratings; switching step count; delta connection and controller indicators; active step fraction.
& -- \\
\addlinespace[0.25em]
\textbf{DER}
& Device kind encoding for PV, generator, or storage; active/reactive power ratings; delta connection indicator.
& -- \\
\midrule
\multicolumn{3}{@{}l}{\textbf{(b) Relations}} \\
\addlinespace[0.25em]
\textbf{Line}\newline bus--bus
& Switch, fuse, and recloser indicators; length; normal/emergency current ratings; sequence capacitances and impedances; phase count. Full phase impedance matrices are optional.
& Per phase current magnitude; active/reactive power flows with extended measurements. \\
\addlinespace[0.25em]
\textbf{Transformer}\newline bus--bus
& Regulator, fuse, and recloser indicators; apparent power rating; turns ratio; winding count; no-load loss, magnetizing current, and interwinding reactance parameters; phase counts. Winding connections and phase presence indicators are optional.
& Terminal voltage, angle, current, and power embeddings with extended measurements. \\
\addlinespace[0.25em]
\textbf{Reactor}\newline bus--bus
& Switch, fuse, and recloser indicators; phase count; optional resistance and reactance.
& -- \\
\addlinespace[0.25em]

\textbf{Device attachments}
& Links to buses from consumers (service), substations (source), capacitors, and DERs.
& -- \\
\addlinespace[0.25em]
\textbf{Communication}\newline (optional)
& Links from substations to DT primary buses and from consumers to DT secondary buses.
& -- \\
\addlinespace[0.25em]
\textbf{Phase head links}\newline (optional)
& Links from consumers to their serving transformer's secondary bus for the phase prediction head.
& -- \\
\addlinespace[0.25em]
\textbf{Terminal attachments}
& Links from terminals to buses and transformer objects.
& -- \\
\midrule
\multicolumn{3}{@{}l}{\textbf{(c) Auxiliary stores: inputs and role}} \\
\addlinespace[0.25em]
\textbf{Transformer terminal}
& \multicolumn{2}{>{\raggedright\arraybackslash}p{\dimexpr0.82\textwidth-2\tabcolsep\relax}@{}}{Winding index; per phase voltage magnitude and angle, current magnitude, active/reactive power, and observability masks.} \\
\addlinespace[0.25em]
\textbf{Transformer object}
& \multicolumn{2}{>{\raggedright\arraybackslash}p{\dimexpr0.82\textwidth-2\tabcolsep\relax}@{}}{Aggregates encoded terminal measurements for fusion into the corresponding transformer edge representation.} \\
\addlinespace[0.25em]
\textbf{Grid / feeder head}
& \multicolumn{2}{>{\raggedright\arraybackslash}p{\dimexpr0.82\textwidth-2\tabcolsep\relax}@{}}{Reference voltage, feeder head active/reactive power, and masks; observed tap summaries and sampling cadence features.} \\
\addlinespace[0.25em]
\textbf{Power flow}
& \multicolumn{2}{>{\raggedright\arraybackslash}p{\dimexpr0.82\textwidth-2\tabcolsep\relax}@{}}{Feeder head current magnitude and mask in the extended measurement path.} \\
\end{longtable}
\endgroup

For entity $u$, let $\mathbf{x}_{u,t}$ denote its measurement vector
and $\mathbf{m}_{u,t}\in\{0,1\}^{d_u}$ its observability mask.
The model receives transformed measurements together with their masks:
\begin{equation}
\widetilde{\mathbf{x}}_{u,t}
=
\left[
\operatorname{where}
\left(
\mathbf{m}_{u,t},
\psi_u(\mathbf{x}_{u,t}),
\mathbf{0}
\right)
\;\Vert\;
\mathbf{m}_{u,t}
\right],
\qquad t=1,\ldots,T,
\end{equation}
where $\psi_u$ denotes the configured channelwise transformation
and $\Vert$ denotes concatenation.
Masks reflect sensor placement and measurement availability.
Channels sharing a measurement source may share a mask.
Unavailable values are zero filled after transformation.
These observation masks are separate from target validity and
task applicability masks used for supervision.

Optional electrical distance encodings supplement the listed
attributes with topology derived position information.
The nominal angle transformation is described in
Appendix~\ref{app:nominal_angle_reference}, and the positional encoding in
Appendix~\ref{app:electrical_position}.

For static attributes $\mathbf{s}_u$ and transformed, masked inputs
$\widetilde{\mathbf{x}}_{u,t}$, type specific node encoders produce
\begin{equation}
\mathbf{H}^{(0)}_u
= E^{\mathrm{node}}_{\tau(u)}
\left(\mathbf{s}_u,\widetilde{\mathbf{x}}_{u,1:T}\right)
\in\mathbb{R}^{T\times d}.
\end{equation}
Analogous edge encoders process physical branch inputs. Parameters are shared
across entities of the same type and across networks. Connectivity only
relations have no separate attribute encoder. Static features are broadcast
across the window, while measurements and masks retain their time indices.

Observation masks identify available inputs. Supervision masks identify valid
targets and applicable tasks. An unobserved voltage can therefore remain a
supervised target when its ground truth is available during training.

\subsection{Communication Edge Construction}
\label{app:communication_graph}

The communication graph augments the electrical graph with two relation types. Uplinks connect the substation to each distinct distribution transformer primary bus. Downlinks connect a customer to the secondary bus of its serving transformer. Both relations support bidirectional message passing and have parameters distinct from the physical relations.

For a customer with a valid transformer association, the augmented graph contains the path
\begin{equation}
\text{substation}
\leftrightarrow\text{primary bus}
\leftrightarrow\text{secondary bus}
\leftrightarrow\text{customer}.
\end{equation}
The middle connection is the physical transformer relation. The first and last connections are shortcuts. Thus, the customer lies within three message passing hops of the feeder head, irrespective of the length of the corresponding electrical route. The shortcuts change information exchange within the model while preserving the physical graph. Their construction uses topology and registry information without adding measurement channels.

\subsection{Learning with Communication Edges}\label{app:communication_theory}
\paragraph{Model and observations.}
Let $N\ge1$ and let buses $0,1,\ldots,N$ form a tree rooted at the fixed voltage source.
Number parents before children and index each branch by its child.
Let $d(i,j)$ count physical branches between buses $i$ and $j$.
All expectations are over $s\sim\mathcal D$, including operating conditions
and measurement noise. Snapshot arguments are suppressed below.
In this appendix, $\mathbf p,\mathbf q$ denote the ground truth injections
in (\ref{eq:communication_lindistflow}), using a net-injection sign convention.
At each bus, active power supplied to the grid contributes positively to $p_i$,
while active power consumption contributes negatively. Likewise, reactive
power injection contributes positively to $q_i$ and absorption negatively. Set
$\delta\mathbf p=\mathbf p-\mathbb{E}[\mathbf p]$, $\delta\mathbf q=\mathbf q-\mathbb{E}[\mathbf q]$,
and $\mathbf{y}=\mathbf V^{\odot2}-\mathbb{E}[\mathbf V^{\odot2}]$.
Since $V_0$ is fixed, subtracting the expectation of (\ref{eq:communication_lindistflow}) gives
$\mathbf{y}=2\mathbf R\delta\mathbf p+2\mathbf X\delta\mathbf q$.
For $\mathbf{z}=(\delta\mathbf p^{\mathsf T},\delta\mathbf q^{\mathsf T})^{\mathsf T}$, assume
$\boldsymbol{\Sigma}=\mathbb{E}[\mathbf{z}\mathbf{z}^{\mathsf T}]\succ0$. The centered measurements are
$\widetilde{\mathbf{z}}=\mathbf{z}+\boldsymbol{\varepsilon}
=(\widetilde{\mathbf{z}}_p^{\mathsf T},\widetilde{\mathbf{z}}_q^{\mathsf T})^{\mathsf T}$,
where $\widetilde{\mathbf{z}}_p,\widetilde{\mathbf{z}}_q\in\mathbb{R}^N$ are its active and reactive
measurement blocks.
The sensor error is independent of $\mathbf{z}$ and satisfies
$\mathbb{E}\boldsymbol{\varepsilon}=\mathbf{0}$ and $\mathbb{E}[\boldsymbol{\varepsilon}\boldsymbol{\varepsilon}^{\mathsf T}]=\tau\boldsymbol{\Sigma}$,
so $\mathbb{E}[\widetilde{\mathbf{z}}\widetilde{\mathbf{z}}^{\mathsf T}]=(1+\tau)\boldsymbol{\Sigma}$.
Each bus receives only its own noisy injection components, with no global
measurement aggregates. Voltage targets are noiseless.
Both injections and sensor errors may be correlated across buses.

Let $J_{ij}=1$ when $j$ is the nonroot parent of $i$, and zero otherwise.
The reduced branch--bus incidence matrix, oriented positively at the
parent, is $\mathbf{A}=\mathbf{J}-\mathbf{I}$. Let the physical depth be $h=\max_i d(0,i)$. Then
$\mathbf{J}^h=\mathbf{0}$ and
\begin{equation}
 \mathbf{T}=-\mathbf{A}^{-1}=\sum_{k=0}^{h-1}\mathbf{J}^k,\qquad
 \mathbf R=\mathbf{T}\operatorname{diag}(\mathbf{r})\mathbf{T}^{\mathsf T},\qquad \mathbf X=\mathbf{T}\operatorname{diag}(\mathbf{x})\mathbf{T}^{\mathsf T}.
 \label{eq:communication_path_factorization}
\end{equation}
The inverse identity follows by multiplying the sum by $\mathbf{I}-\mathbf{J}$.
These factorizations yield the radial LinDistFlow relation
(\ref{eq:communication_lindistflow}) \citep{kekatos2015fast,deka2024learning}, assumed exact here.
The vectors $\mathbf{r},\mathbf{x}$ collect branch resistances $r_e>0$ and reactances $x_e\ge0$.
Direct multiplication gives
$R_{ij}=\sum_{e\in\mathcal{P}_i\cap\mathcal{P}_j}r_e$ and
$X_{ij}=\sum_{e\in\mathcal{P}_i\cap\mathcal{P}_j}x_e$, where $\mathcal{P}_i$ is the root path
of bus $i$. An injection at $j$ affects squared voltage at $i$ when
$R_{ij}$ or $X_{ij}$ is nonzero.

\paragraph{Learner and communication.}
The model learns unconstrained branch coefficients $\mathbf{w}=(\mathbf{w}_r^{\mathsf T},\mathbf{w}_x^{\mathsf T})^{\mathsf T}$:
\begin{equation}
 \widehat{\mathbf{y}}_{\mathbf{w}}
 =\mathbf{T}\bigl[\operatorname{diag}(\mathbf{w}_r)\mathbf{T}^{\mathsf T}\widetilde{\mathbf{z}}_p
       +\operatorname{diag}(\mathbf{w}_x)\mathbf{T}^{\mathsf T}\widetilde{\mathbf{z}}_q\bigr]
 =\boldsymbol{\Phi}(\widetilde{\mathbf{z}})\mathbf{w},
 \label{eq:communication_linear_learner}
\end{equation}
For any $\mathbf a=(\mathbf a_p^{\mathsf T},\mathbf a_q^{\mathsf T})^{\mathsf T}$,
$\boldsymbol{\Phi}(\mathbf a)=\mathbf T[\,\operatorname{diag}(\mathbf T^{\mathsf T}\mathbf a_p)\;\;\operatorname{diag}(\mathbf T^{\mathsf T}\mathbf a_q)\,]$,
which is linear in its argument. The physical coefficients
$\mathbf{w}_*=(2\mathbf{r}^{\mathsf T},2\mathbf{x}^{\mathsf T})^{\mathsf T}$ satisfy $\mathbf{y}=\boldsymbol{\Phi}(\mathbf{z})\mathbf{w}_*$.
The factor two comes from squared voltage magnitudes.
Coefficients are shared across snapshots
and affected outputs, but remain branch specific.
The main text prediction is
$\widehat{\mathbf V^{\odot2}}_{\mathbf{w}}=\mathbb{E}[\mathbf V^{\odot2}]+\widehat{\mathbf{y}}_{\mathbf{w}}$,
an affine map of the raw injections for fixed weights.

Set $K=\lceil\log_2 h\rceil$. Include bidirectional links to every
nonroot ancestor at physical distances $2^k$, $0\le k<K$. Then
\begin{equation}
 \prod_{k=0}^{K-1}(\mathbf{I}+\mathbf{J}^{2^k})
 =\sum_{j=0}^{2^K-1}\mathbf{J}^j=\mathbf{T}.
 \label{eq:communication_doubling}
\end{equation}
Each exponent occurs once by binary expansion, and powers at least $h$
vanish. The factors commute. Applying their transposes to
$[\,\widetilde{\mathbf{z}}_p\;\;\widetilde{\mathbf{z}}_q\,]$
computes observed subtree injection sums in $K$ upward rounds. Bus $e$ multiplies
the two sums by $w_{r,e},w_{x,e}$ and adds them. Applying the original
factors accumulates these contributions along root paths in $K$ downward
rounds, giving (\ref{eq:communication_linear_learner}). Messages have two coordinates upward
and one downward, and all aggregation weights are fixed at one.
For $h=1$, the products are empty and communication is unnecessary.

\begin{proof}[Proof of Proposition~\ref{prop:communication_learning}]
Centering (\ref{eq:communication_loss}) gives
$\mathcal L(\mathbf{w})=\tfrac12\mathbb{E}\left\lVert \widehat{\mathbf{y}}_{\mathbf{w}}-\mathbf{y}\right\rVert^2$.
Write $\mathbf{M}_*=[\,2\mathbf R\;\;2\mathbf X\,]$, so $\mathbf{y}=\mathbf{M}_*\mathbf{z}$, and put
$\alpha=(1+\tau)^{-1}$. Define the residual
$\boldsymbol{\xi}=\mathbf{y}-\alpha \mathbf{M}_*\widetilde{\mathbf{z}}
=\mathbf{M}_*[(1-\alpha)\mathbf{z}-\alpha\boldsymbol{\varepsilon}]$.
The noise assumptions give
\begin{align}
 \mathbb{E}[\boldsymbol{\xi}\widetilde{\mathbf{z}}^{\mathsf T}]
 &=\mathbf{M}_*[(1-\alpha)\boldsymbol{\Sigma}-\alpha\tau\boldsymbol{\Sigma}]=\mathbf{0}, \label{eq:communication_noise_orthogonality}\\
 \nu_\tau:=\mathbb{E}\left\lVert \boldsymbol{\xi}\right\rVert^2
 &=\bigl[(1-\alpha)^2+\alpha^2\tau\bigr]\operatorname{tr}(\mathbf{M}_*\boldsymbol{\Sigma} \mathbf{M}_*^{\mathsf T})
  =(1-\alpha)\mathcal{E}_{\rm H}(0). \label{eq:communication_noise_floor}
\end{align}
The last equality uses
$\mathcal{E}_{\rm H}(0)=\mathbb{E}\left\lVert \mathbf{y}\right\rVert^2=\operatorname{tr}(\mathbf{M}_*\boldsymbol{\Sigma} \mathbf{M}_*^{\mathsf T})$.
For any linear predictor $\mathbf{F}\widetilde{\mathbf{z}}$, expand its error as
$(\mathbf{F}-\alpha \mathbf{M}_*)\widetilde{\mathbf{z}}-\boldsymbol{\xi}$.
The cross term vanishes by (\ref{eq:communication_noise_orthogonality}), yielding
\begin{equation}
 \mathbb{E}\left\lVert \mathbf{F}\widetilde{\mathbf{z}}-\mathbf{y}\right\rVert^2
 =\nu_\tau+(1+\tau)\operatorname{tr}\bigl[(\mathbf{F}-\alpha \mathbf{M}_*)\boldsymbol{\Sigma}(\mathbf{F}-\alpha \mathbf{M}_*)^{\mathsf T}\bigr].
 \label{eq:communication_error_decomposition}
\end{equation}
Thus $\nu_\tau$ is the minimum error among unrestricted linear predictors.
Since $\alpha \mathbf{M}_*$ is realized by $\mathbf{w}_\tau=\alpha \mathbf{w}_*$, the same minimum
is attainable by our learner. No Gaussian assumption is used.

Let $\mathbf{H}_0=\mathbb{E}[\boldsymbol{\Phi}(\mathbf{z})^{\mathsf T}\boldsymbol{\Phi}(\mathbf{z})]$ and
$\mathbf{H}_\tau=\mathbb{E}[\boldsymbol{\Phi}(\widetilde{\mathbf{z}})^{\mathsf T}\boldsymbol{\Phi}(\widetilde{\mathbf{z}})]$.
Each entry of $\boldsymbol{\Phi}$ is linear in its argument, so zero cross covariance
between $\mathbf{z},\boldsymbol{\varepsilon}$ and the proportional noise covariance give
$\mathbf{H}_\tau=(1+\tau)\mathbf{H}_0$. Equation~(\ref{eq:communication_error_decomposition}) implies
\begin{equation}
 \mathcal L(\mathbf w)=\frac{\nu_\tau}{2}
 +\frac12(\mathbf w-\mathbf w_\tau)^{\mathsf T}
      \mathbf H_\tau(\mathbf w-\mathbf w_\tau).
 \label{eq:communication_quadratic_loss}
\end{equation}
For $\mathbf u=(\mathbf u_r^{\mathsf T},\mathbf u_x^{\mathsf T})^{\mathsf T}$, set
$\mathbf M(\mathbf u)=[\,\mathbf T\operatorname{diag}(\mathbf u_r)\mathbf T^{\mathsf T}\;\;
\mathbf T\operatorname{diag}(\mathbf u_x)\mathbf T^{\mathsf T}\,]$.
Then $\mathbf u^{\mathsf T}\mathbf H_\tau\mathbf u
=(1+\tau)\operatorname{tr}(\mathbf M(\mathbf u)\boldsymbol\Sigma\mathbf M(\mathbf u)^{\mathsf T})$.
Since $\mathbf T$ is invertible, $\mathbf M(\mathbf u)\ne\mathbf0$ for
$\mathbf u\ne\mathbf0$. Positive definiteness of $\boldsymbol\Sigma$ makes
this trace strictly positive. Hence $\mathbf H_\tau\succ0$ and all its
eigenvalues are positive.

For $\mathbf{e}_n=\mathbf{w}^{(n)}-\mathbf{w}_\tau$, gradient descent gives
$\mathbf{e}_n=(\mathbf{I}-\eta' \mathbf{H}_\tau)^n \mathbf{e}_0$, with $\mathbf{e}_0=-\mathbf{w}_\tau$.
Diagonalize $\mathbf{H}_\tau=\mathbf{U}\operatorname{diag}(\lambda_j)\mathbf{U}^{\mathsf T}$, where $\mathbf{U}^{\mathsf T} \mathbf{U}=\mathbf{I}$,
and set $\mathbf{c}=\mathbf{U}^{\mathsf T} \mathbf{e}_0$. Then
\begin{equation}
 \mathcal E_{\rm H}(n)-\nu_\tau
 =\sum_j\lambda_j(1-\eta'\lambda_j)^{2n}c_j^2
 =:Q_n.
 \label{eq:communication_gd_error}
\end{equation}
This is an equality for the population loss, with no bound on the Hessian
spectrum substituted. At initialization,
$Q_0=\mathbf e_0^{\mathsf T}\mathbf H_\tau\mathbf e_0
=\mathcal E_{\rm H}(0)-\nu_\tau=\alpha\mathcal E_{\rm H}(0)$.
For $0<\eta'<2/\max_j\lambda_j$, every factor
$|1-\eta'\lambda_j|$ is less than one, so $Q_n$ is nonincreasing and tends to zero.

For bus $i$ and integer budget $\ell\ge0$, define the accessible and
inaccessible injection coordinates
\begin{equation}
 \mathcal A_i=\{j,N+j:1\le j\le N,\ d(i,j)\le\ell\},\qquad
 \mathcal U_i=\{1,\ldots,2N\}\setminus\mathcal A_i.
 \label{eq:communication_accessible_coordinates}
\end{equation}
The two coordinates for each bus correspond to active and reactive
injections. Subscripts on $\boldsymbol\Sigma$ select the indicated rows
and columns, in a fixed order. Write
$\mathbf m_i=(\mathbf M_*)_{i,:}\in\mathbb R^{1\times2N}$ for row $i$.
Induction over communication rounds shows that every physical linear
predictor has $\mathbf F_{i,\mathcal U_i}=\mathbf0$.
For nonempty $\mathcal U_i$, set
\begin{equation}
 \mathbf C_i
 =\boldsymbol\Sigma_{\mathcal U_i\mathcal U_i}
  -\boldsymbol\Sigma_{\mathcal U_i\mathcal A_i}
   \boldsymbol\Sigma_{\mathcal A_i\mathcal A_i}^{-1}
   \boldsymbol\Sigma_{\mathcal A_i\mathcal U_i},\qquad
 C_\ell=\sum_{i:\mathcal U_i\ne\emptyset}
   \mathbf m_{i,\mathcal U_i}\mathbf C_i\mathbf m_{i,\mathcal U_i}^{\mathsf T}.
 \label{eq:communication_exact_local_floor}
\end{equation}
The principal block being inverted is positive definite because
$\boldsymbol\Sigma\succ0$ and $\mathcal A_i$ includes bus $i$.
An empty inaccessible set contributes zero.

To minimize (\ref{eq:communication_error_decomposition}) subject to locality,
let $\mathbf d_i=\mathbf F_{i,\mathcal A_i}-\alpha\mathbf m_{i,\mathcal A_i}$.
The corresponding row of $\mathbf F-\alpha\mathbf M_*$ is
$[\,\mathbf d_i,-\alpha\mathbf m_{i,\mathcal U_i}\,]$
in the $(\mathcal A_i,\mathcal U_i)$ ordering. Its quadratic contribution is
\begin{align}
 &\mathbf d_i\boldsymbol\Sigma_{\mathcal A_i\mathcal A_i}\mathbf d_i^{\mathsf T}
 -2\alpha\mathbf d_i\boldsymbol\Sigma_{\mathcal A_i\mathcal U_i}
                     \mathbf m_{i,\mathcal U_i}^{\mathsf T}
 +\alpha^2\mathbf m_{i,\mathcal U_i}
              \boldsymbol\Sigma_{\mathcal U_i\mathcal U_i}
              \mathbf m_{i,\mathcal U_i}^{\mathsf T}\nonumber\\
 &\quad=\boldsymbol\delta_i\boldsymbol\Sigma_{\mathcal A_i\mathcal A_i}
                    \boldsymbol\delta_i^{\mathsf T}
       +\alpha^2\mathbf m_{i,\mathcal U_i}\mathbf C_i
                    \mathbf m_{i,\mathcal U_i}^{\mathsf T},\qquad
 \boldsymbol\delta_i=\mathbf F_{i,\mathcal A_i}-\mathbf F^\star_{i,\mathcal A_i},
 \label{eq:communication_complete_square}
\end{align}
where
\begin{equation}
 \mathbf F^\star_{i,\mathcal A_i}
 =\alpha\left(\mathbf m_{i,\mathcal A_i}
   +\mathbf m_{i,\mathcal U_i}
      \boldsymbol\Sigma_{\mathcal U_i\mathcal A_i}
      \boldsymbol\Sigma_{\mathcal A_i\mathcal A_i}^{-1}\right),\qquad
 \mathbf F^\star_{i,\mathcal U_i}=\mathbf0.
 \label{eq:communication_optimal_local_predictor}
\end{equation}
If $\mathcal U_i$ is empty, the second term in parentheses is absent and
$\mathbf F^\star_{i,:}=\alpha\mathbf m_i$.
Using $(1+\tau)\alpha^2=\alpha$ and summing the completed squares gives
\begin{equation}
 \mathcal E_{\rm P}(\mathbf F)
 =\nu_\tau+\alpha C_\ell
 +(1+\tau)\sum_i\boldsymbol\delta_i
       \boldsymbol\Sigma_{\mathcal A_i\mathcal A_i}\boldsymbol\delta_i^{\mathsf T}.
 \label{eq:communication_physical_floor}
\end{equation}
Thus the minimum over all local linear maps is
$\mathcal E_{\rm P}^\star(\ell)=\nu_\tau+\alpha C_\ell$, attained uniquely by
$\mathbf F^\star$ on the allowed coordinates.
This minimum is also attainable with $\ell$ physical communication rounds
when message width and node specific coefficients are unrestricted.
To see this, place each node's two measurements in their own coordinates
of a $2N$-vector. In the first round, send that vector to each neighbor.
In later rounds, forward the preceding round's vectors to all neighbors
except the edge from which each arrived. On a tree, a measurement reaches
another bus exactly once, after its physical distance in rounds.
Summing the received vectors over rounds and retaining the bus's own
coordinates therefore supplies exactly $\widetilde{\mathbf z}_{\mathcal A_i}$.
All routing and aggregation are linear. Bus $i$ then applies row
$\mathbf F^\star_{i,\mathcal A_i}$. The root can relay messages with zero
initial coordinates. This argument establishes attainability in the
unrestricted comparison class, not in a prescribed fixed width or
shared weight neural architecture.

Combining (\ref{eq:communication_gd_error}) and
(\ref{eq:communication_physical_floor}) yields the exact decomposition
\begin{equation}
 \boxed{\displaystyle
 \mathcal E_{\rm P}(\mathbf F)-\mathcal E_{\rm H}(n)
 =\frac{C_\ell}{1+\tau}-Q_n
 +(1+\tau)\sum_i\boldsymbol\delta_i
       \boldsymbol\Sigma_{\mathcal A_i\mathcal A_i}\boldsymbol\delta_i^{\mathsf T}.}
 \label{eq:communication_exact_gap}
\end{equation}
The three terms describe inaccessible information, remaining GD error,
and the physical predictor's excess error above its local optimum.
Dropping only the last, nonnegative term gives a sharp lower bound
for each fixed grid, covariance, and iterate: equality holds exactly
when $\mathbf F=\mathbf F^\star$.
Setting $\mathbf F=\mathbf F^\star$ and substituting the exact expression
for $Q_n$ proves Proposition~\ref{prop:communication_learning}.
Every other physical linear predictor has at least this error gap.
Linearity is stated after centering. Restoring the
true output mean gives affine predictors of the raw observations.
Any additional constant bias contributes its squared norm to the risk.
\end{proof}

\paragraph{Why the Schur complement is the relevant covariance.}
The residual after the best linear prediction of inaccessible injections is
\begin{equation}
 \boldsymbol\zeta_i
 =\mathbf z_{\mathcal U_i}
  -\boldsymbol\Sigma_{\mathcal U_i\mathcal A_i}
   \boldsymbol\Sigma_{\mathcal A_i\mathcal A_i}^{-1}\mathbf z_{\mathcal A_i}.
 \label{eq:communication_linear_residual}
\end{equation}
Direct expansion gives
$\mathbb E[\boldsymbol\zeta_i\mathbf z_{\mathcal A_i}^{\mathsf T}]=\mathbf0$
and $\mathbb E[\boldsymbol\zeta_i\boldsymbol\zeta_i^{\mathsf T}]=\mathbf C_i$.
Orthogonality makes every other linear reconstruction add a nonnegative
quadratic error. Thus $C_\ell$ sums the residual variances after weighting
by the corresponding voltage sensitivities.
No Gaussian assumption is needed. For a general distribution,
$\mathbf C_i$ is a linear residual covariance, not necessarily a
conditional covariance.

\paragraph{Positivity and the exact step threshold.}
For a nonzero column vector $\mathbf b$ on a nonempty $\mathcal U_i$, set
$\mathbf a_*=-\boldsymbol\Sigma_{\mathcal A_i\mathcal A_i}^{-1}
\boldsymbol\Sigma_{\mathcal A_i\mathcal U_i}\mathbf b$.
Expanding the quadratic form gives
\begin{equation}
 \mathbf b^{\mathsf T}\mathbf C_i\mathbf b
 =\begin{bmatrix}\mathbf a_*\\\mathbf b\end{bmatrix}^{\mathsf T}
  \begin{bmatrix}
   \boldsymbol\Sigma_{\mathcal A_i\mathcal A_i}&\boldsymbol\Sigma_{\mathcal A_i\mathcal U_i}\\
   \boldsymbol\Sigma_{\mathcal U_i\mathcal A_i}&\boldsymbol\Sigma_{\mathcal U_i\mathcal U_i}
  \end{bmatrix}
  \begin{bmatrix}\mathbf a_*\\\mathbf b\end{bmatrix}>0.
 \label{eq:communication_schur_positive}
\end{equation}
Thus $\mathbf C_i\succ0$, so $C_\ell>0$ exactly when some buses $i,j$
more than $\ell$ physical hops apart satisfy $R_{ij}^2+X_{ij}^2>0$.
Increasing $\ell$ enlarges the feasible predictor class, so $C_\ell$
cannot increase. Choosing the zero predictor in the noiseless problem
gives $0\le C_\ell\le\mathcal E_{\rm H}(0)$.

For $C_\ell>0$, the exact first GD step that beats every physical linear
predictor is
\begin{equation}
 n_\sharp=\min\{n\in\mathbb N_0:Q_n<C_\ell/(1+\tau)\}.
 \label{eq:communication_exact_threshold}
\end{equation}
Since $Q_n$ is nonincreasing and tends to zero under the stated step size condition,
$n_\sharp$ is finite, and every later iterate also satisfies the strict
comparison. If $C_\ell=0$, an optimal physical predictor already attains
$\nu_\tau$, so no hierarchical linear predictor can strictly outperform it.

\paragraph{Effect of the noise level.}
The base step size $\eta$ is constant across GD iterations.
Since $\mathbf H_\tau=(1+\tau)\mathbf H_0$ and
$\mathbf w_\tau=\alpha\mathbf w_*$, the choice
$\eta'=\eta/(1+\tau)$ gives
$\mathbf I-\eta'\mathbf H_\tau=\mathbf I-\eta\mathbf H_0$.
With zero initialization, the noisy parameter error is therefore
$\alpha$ times the noiseless error at every iteration.
Its squared Hessian norm satisfies $Q_n(\tau)=\alpha Q_n(0)$.
Because $C_\ell$ is independent of $\tau$, the exact optimal comparator
error gap in (\ref{eq:communication_prediction_gap}) equals its noiseless value
divided by $1+\tau$. Thus the exact crossing time is unchanged by
$\tau$ under this matched step scaling. The common noise floor still
increases with $\tau$.

\paragraph{Resistance uncertainty.}
Keep topology, injection covariance, and noise level fixed. Let
$\boldsymbol\kappa=(\mathbf r^{\mathsf T},\mathbf x^{\mathsf T})^{\mathsf T}$
range over the compact box $\mathcal B$ defined by
$r_e\in[\underline r_e,\overline r_e]$ and
$x_e\in[\underline x_e,\overline x_e]$, with
$\underline r_e>0$ and $\underline x_e\ge0$.
The Hessian $\mathbf H_\tau$ and its eigenbasis remain fixed, while
$C_\ell$ and the initial error coordinates depend on $\boldsymbol\kappa$.
For a fixed admissible step size, Proposition~\ref{prop:communication_learning}
gives the exact worst-case gap
\begin{equation}
 \Gamma_{\ell,n}
 =\min_{\boldsymbol\kappa\in\mathcal B}
   \left\{\frac{C_\ell(\boldsymbol\kappa)}{1+\tau}
             -Q_n(\boldsymbol\kappa)\right\}.
 \label{eq:communication_uniform_exact_gap}
\end{equation}
Continuity and compactness ensure that the minimum is attained.
Thus $\Gamma_{\ell,n}>0$ holds exactly when the hierarchical learner
strictly outperforms every physical linear predictor at every impedance
assignment in the box. The learners are retrained from zero for each
assignment. Both terms are quadratic in $\boldsymbol\kappa$: the voltage
sensitivities and initial parameter error are linear in the impedances,
while the residual covariances and Hessian are fixed.
Their difference need not be convex, so this exact uniform comparison
is a potentially nonconvex quadratic minimization over the box.
Substituting only the lower impedance endpoints is not generally valid.

\paragraph{Scope.}
The result requires the prescribed links in both directions and noisy
injection coverage at every bus. It concerns population loss under exact
LinDistFlow, with clean voltage targets.
For a general noise covariance $\boldsymbol{\Omega}$, the best linear response is
$\mathbf{M}_*\boldsymbol{\Sigma}(\boldsymbol{\Sigma}+\boldsymbol{\Omega})^{-1}$, obtained from the normal equation
$\mathbf{F}(\boldsymbol{\Sigma}+\boldsymbol{\Omega})=\mathbf{M}_*\boldsymbol{\Sigma}$.
It need not retain the learner's path factorization.
Thus the proportional covariance assumption is substantive.
The floor $\nu_\tau$ is optimal among linear predictors, not
necessarily among nonlinear predictors. Gaussianity is not assumed.
For finite samples the same proof requires the corresponding empirical
moment identities, not merely a positive definite empirical covariance.
It does not by itself establish a test error guarantee. The exact physical optimum uses population covariances and unrestricted local linear coefficients.
If physical communication already implements (\ref{eq:communication_linear_learner}),
the two implementations have identical losses and gradient iterates.

\subsection{Nominal Angle Reference}
\label{app:nominal_angle_reference}

We separate static phase conventions from operating angle deviations using a nominal reference $\theta^{\mathrm{nom}}_{i,p}$ for each (bus, phase) pair. The reference is derived from network topology, source configuration, and transformer metadata, without using simulated voltage trajectories. A validity mask $r_{i,p}$ identifies pairs with a defined reference.
Source phase angles initialize propagation through the network. Lines and reactors preserve phase identity with zero nominal shift, leaving angle drops that depend on operating conditions to the residual prediction. Transformer mappings follow winding conductor order and connection metadata:
\begin{itemize}
    \item Matching winding connection types assign zero shift to corresponding conductor positions.
    \item Mixed delta-wye connections with three matched conductor positions introduce a shift of $\pm\pi/6$, determined by the configured lead/lag convention.
    \item Single-conductor winding mappings assign zero shift to matching positions and $\pi$ to opposite positions, accounting for polarity and center-tapped secondary legs.
\end{itemize}
Mappings include only phases present on the branch and its endpoint buses. Unsupported mappings are omitted, and conflicting mappings for the same phase pair are discarded. Reverse traversal inverts the phase mapping and negates its shift.
All angles below are in radians, with
\begin{equation}
\operatorname{wrap}(x)=((x+\pi)\bmod 2\pi)-\pi.
\label{eq:app_wrap}
\end{equation}
For valid (bus, phase) pairs, observed angles and targets use the same residual coordinates. The state head predicts $\widehat{\Delta\theta}_{i,p,t}$ and recovers absolute angles through
\begin{equation}
\begin{aligned}
\Delta\theta_{i,p,t}
&=\operatorname{wrap}\left(\theta_{i,p,t}-\theta^{\mathrm{nom}}_{i,p}\right),\\
\widehat{\theta}_{i,p,t}
&=\operatorname{wrap}\left(\theta^{\mathrm{nom}}_{i,p}
+\widehat{\Delta\theta}_{i,p,t}\right).
\end{aligned}
\label{eq:app_angle_residual}
\end{equation}
Because prediction and target share the reference,
\begin{equation}
\operatorname{wrap}(\widehat{\theta}_{i,p,t}-\theta_{i,p,t})
=\operatorname{wrap}(\widehat{\Delta\theta}_{i,p,t}-\Delta\theta_{i,p,t}).
\end{equation}
This parameterization removes static phase offsets from the regression target while preserving circular angle error. Figure~\ref{fig:nominal_angle_reference} illustrates the nominal references and corresponding operating angle residuals. Appendix~\ref{app:se_readout} specifies the state decoder.

\begin{figure}[t]
\centering
\includegraphics[width=\linewidth]{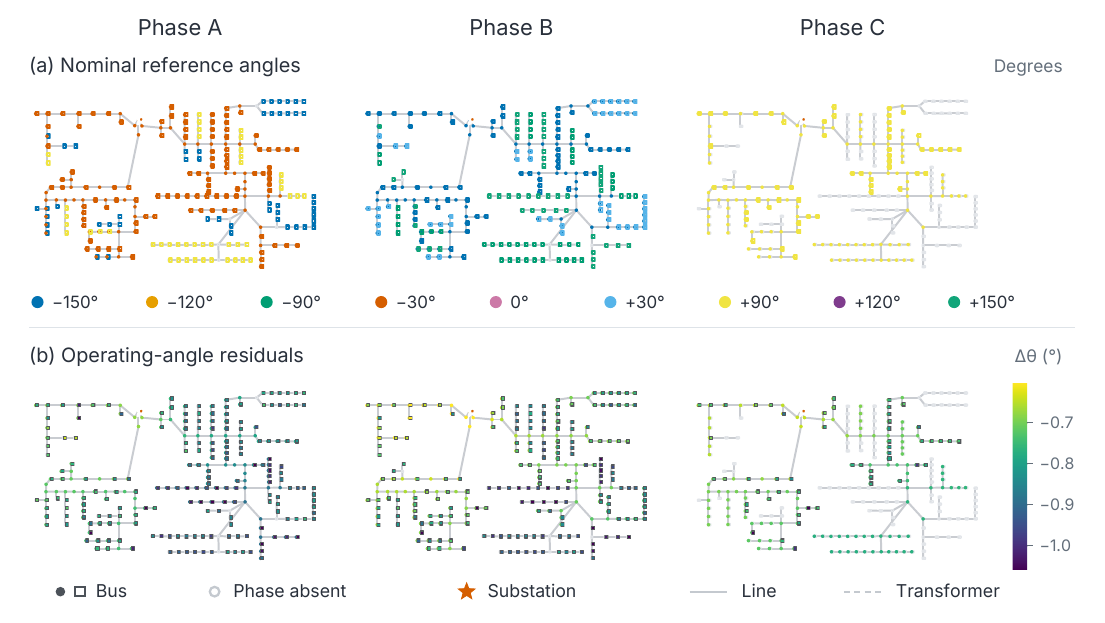}
\caption{Nominal angle parameterization on an example network. Columns show phases A, B, and C. (a) Nominal references $\theta^{\mathrm{nom}}$ encode source and transformer phase conventions. (b) Residuals $\Delta\theta=\operatorname{wrap}(\theta-\theta^{\mathrm{nom}})$ capture operating angle deviations. Nominal values use categorical colors, and residuals share a continuous scale. Light gray markers indicate absent phases, and solid and dashed edges denote lines and transformers, respectively. Angles are shown in degrees.}
\label{fig:nominal_angle_reference}
\end{figure}

\subsection{Electrical Distance Positional Encoding}
\label{app:electrical_position}

We construct positional descriptors on the full structural graph
of lines, transformers, and reactors. All switches are treated as
closed, independently of their operational state. The descriptors
therefore do not use switch state labels and are computed once
per network and physical topology variant.

For each branch $e$, resistance and reactance are converted to
per unit values using a fixed $1$~MVA system base and the local
voltage base. Transformer impedances use the primary side base.
Line impedances are obtained from scalar sequence style summaries
of their phase impedance matrices.
The Dijkstra edge weight is
\begin{equation}
w_e=\sqrt{(r_e^{\mathrm{pu}})^2+(x_e^{\mathrm{pu}})^2}.
\end{equation}
For bus $i$, the selected substation rooted path is
\begin{equation}
\mathcal{P}_i
=
\arg\min_{\mathcal{P}:s\leadsto i}
\sum_{e\in\mathcal{P}}w_e.
\end{equation}
Along this path, we record the hop count $\ell_i$, accumulated
resistance $R_i$, accumulated reactance $X_i$, transformer count
$n_i^{\mathrm{tr}}$, cumulative phase shift $\psi_i$, and number
of nonzero shift transformer crossings $n_i^{\mathrm{shift}}$.
Transformer shifts respect traversal direction. A transformer
contributes a phase shift only when its resolved phase mappings
agree on a single value. Otherwise its scalar contribution is zero.

The nine-dimensional descriptor is
\begin{equation}
\begin{aligned}
\mathbf{a}_i^{\mathrm{pos}}=\big[&
\ell_i/D_{\max},\;\log(1+\ell_i),\;
\log(1+R_i/R_0),\;\log(1+X_i/X_0),\\
&\log(1+n_i^{\mathrm{tr}}),\;
\log(V_i^{\mathrm{base}}/V_s^{\mathrm{base}}),\;
\sin\psi_i,\;\cos\psi_i,\;
\log(1+n_i^{\mathrm{shift}})\big]^{\!\top}.
\end{aligned}
\end{equation}
We use corpus calibrated scales
$R_0=0.66$, $X_0=0.54$, and $D_{\max}=57$.
Unreachable buses receive zero valued raw path statistics.

A shared MLP and LayerNorm map the descriptor to
\begin{equation}
\mathbf{p}_i
=
\operatorname{LN}
\left(\operatorname{MLP}_{\mathrm{PE}}(\mathbf{a}_i^{\mathrm{pos}})\right)
\in\mathbb{R}^{16}.
\end{equation}
This embedding is concatenated with each bus's static attributes
and masked measurement inputs before the bus encoder:
\begin{equation}
\mathbf{h}_{i,t}^{(0)}
=
\operatorname{Enc}_{\mathrm{bus}}
\left(
\left[
\mathbf{s}_i
\;\Vert\;
\widetilde{\mathbf{x}}_{i,t}
\;\Vert\;
\mathbf{p}_i
\right]
\right).
\end{equation}
The same $\mathbf{p}_i$ is used at every timestep.

\subsection{Heterogeneous Graph Transformer}
\label{app:hgt_backbone}

\subsubsection{Residual Updates}

The attention output is followed by a residual update and a
type specific feedforward sublayer:
\begin{align}
\mathbf{u}_{i,t}^{(\ell)}
&=
\mathbf{h}_{i,t}^{(\ell)}
+\operatorname{Dropout}\left(\mathbf{m}_{i,t}^{(\ell)}\right),
\\
\mathbf{h}_{i,t}^{(\ell+1)}
&=
\mathbf{u}_{i,t}^{(\ell)}
+\operatorname{Dropout}\left[
\operatorname{FFN}_{\tau(i)}^{(\ell)}
\left(
\operatorname{LN}_{\tau(i)}^{(\ell,\mathrm{ffn})}
(\mathbf{u}_{i,t}^{(\ell)})
\right)
\right].
\end{align}
Each FFN has the form
\begin{equation}
\operatorname{FFN}_{\tau}^{(\ell)}(\mathbf{z})
=
\mathbf{W}_{2,\tau}^{(\ell)}
\operatorname{Dropout}
\left(
\operatorname{GELU}
\left(\mathbf{W}_{1,\tau}^{(\ell)}\mathbf{z}+\mathbf{b}_{1,\tau}^{(\ell)}\right)
\right)
+\mathbf{b}_{2,\tau}^{(\ell)},
\end{equation}
with intermediate width $2d$.
The attention operator has no additional internal root projection
in its output, leaving a single residual pathway around each
sublayer. After the final layer, a type specific LayerNorm produces
the representations supplied to the task readouts.

\subsubsection{Preserving the Temporal Dimension}

Type specific input encoders produce
\[
\mathbf{H}_{\tau}^{(0)}
\in\mathbb{R}^{|\mathcal{V}_{\tau}|\times T\times d}.
\]
Static attributes are broadcast across the window, while measurements
and observability masks retain their timestep indices.
Every spatial layer preserves this shape:
\begin{equation}
\mathbf{H}_{:,t}^{(\ell+1)}
=
F^{(\ell)}
\left(
\mathbf{H}_{:,t}^{(\ell)},
\mathbf{E}_{:,t},
\mathcal{G}
\right),
\qquad t=1,\ldots,T.
\end{equation}
The same $F^{(\ell)}$ is applied at every timestep, with no
cross time attention or temporal pooling inside the spatial stack.
We implement these operations by batching the snapshots as disjoint
graphs and restoring the explicit time axis afterward.

Consequently, the backbone outputs a sequence of spatially informed
representations for every node. Temporal modeling and aggregation
are performed by the task specific readouts, allowing tasks to use
different summaries of the same observation window.

\begin{table}[t]
\centering
\caption{Mycelium configuration. Parameter count includes the task heads.}
\label{tab:backbone_config}
\begin{tabular}{ll}
\toprule
Component & Setting \\
\midrule
Transformer layers & 8 \\
Hidden dimension & 128 \\
Attention heads & 4 \\
Node type FFN hidden width & 256 \\
Temporal window & 24 hourly steps \\
Positional encoding & $9\rightarrow16$ dimensions \\
Normalization & Pre-LN and final LN \\
Dropout & 0.1 \\
Total parameters & Approximately $12.2$M \\
\bottomrule
\end{tabular}
\end{table}

Let $\mathbf{h}_{i,t}\in\mathbb{R}^{d}$ denote a node's final
backbone representation and $\mathbf{e}_{ij,t}\in\mathbb{R}^{d}$
an encoded branch, with $d=128$.
Each task has independently parameterized temporal readouts.

\subsubsection{Context Construction}

For branch $e=(i,j)$, the readout input combines both endpoints
and the branch embedding:
\begin{equation}
\mathbf{x}_{e,t}
=
[\mathbf{h}_{i,t}\Vert\mathbf{h}_{j,t}
\Vert\mathbf{e}_{ij,t}].
\end{equation}
The network representation combines encoded feeder head measurements
$\mathbf{u}_{\mathcal{G},t}$, mean pooled bus embeddings, and the
substation embedding:
\begin{equation}
\mathbf{g}_{\mathcal{G},t}
=
f_{\mathcal{G}}\left(
\left[
\mathbf{u}_{\mathcal{G},t}
\;\Vert\;
\operatorname{mean}_{i\in\mathcal{V}_{\mathrm{bus}}}
\mathbf{h}_{i,t}
\;\Vert\;
\mathbf{h}_{s,t}
\right]\right).
\end{equation}
For consumer $c$ attached to bus $b(c)$, we use
\begin{equation}
\mathbf{x}_{c,t}
=
[\mathbf{h}_{c,t}\Vert\mathbf{h}_{b(c),t}
\Vert\mathbf{g}_{\mathcal{G},t}].
\end{equation}

\subsection{Task Readout}
\label{app:task_readouts}
\subsubsection{State Estimation}
\label{app:se_readout}

The state estimation readout applies four residual blocks of causal dilated temporal convolutions, each with kernel size three and dilation
$\delta_b\in\{1,2,4,8\}$:
\begin{align}
\mathbf{z}_{i,t}^{(0)}
&=\mathbf{h}_{i,t},\\
\mathbf{z}_{i,t}^{(b+1)}
&=
\operatorname{LN}_b\left(
\mathbf{z}_{i,t}^{(b)}
+
\operatorname{Dropout}\left(
\operatorname{SiLU}\left(
\mathbf{b}_b+
\sum_{q=0}^{2}
\mathbf{W}_{b,q}\mathbf{z}_{i,t-q\delta_b}^{(b)}
\right)\right)\right).
\end{align}
Left zero padding preserves sequence length and ensures that the
readout accesses only current and preceding backbone representations.
This establishes causality of the temporal refiner. End-to-end causality also
depends on preprocessing and input encoding.

A shared linear map produces three magnitude and three angle residual
outputs per bus and timestep:
\begin{equation}
[\mathbf{a}_{i,t}\Vert\mathbf{b}_{i,t}]
=
\mathbf{W}_{\mathrm{SE}}\mathbf{z}_{i,t}^{(4)}
+\mathbf{b}_{\mathrm{SE}},
\qquad
\mathbf{a}_{i,t},\mathbf{b}_{i,t}\in\mathbb{R}^{3}.
\end{equation}
The outputs are converted to per unit voltage magnitudes and
deviations from the nominal angles defined in Appendix~\ref{app:nominal_angle_reference}:
\begin{equation}
\widehat{|V|}_{i,p,t}^{\mathrm{pu}}
=
1+\frac{a_{i,p,t}}{100},
\qquad
\widehat{\Delta\theta}_{i,p,t}
=
\pi\tanh(b_{i,p,t}).
\end{equation}
We recover absolute angles by adding the nominal reference,
\begin{equation}
\widehat{\theta}_{i,p,t}
=
\operatorname{wrap}
\left(
\theta_{i,p}^{\mathrm{nom}}
+
\widehat{\Delta\theta}_{i,p,t}
\right).
\end{equation}
Observed angle inputs and angle targets are likewise expressed as
wrapped deviations from $\theta_{i,p}^{\mathrm{nom}}$.

For a qualifying center-tapped secondary, the angle of one leg is
derived from the other:
\begin{equation}
\widehat{\theta}_{i,\mathrm{derived},t}
=
\operatorname{wrap}
\left(
\widehat{\theta}_{i,\mathrm{reference},t}+\pi
\right).
\end{equation}
Here, ``reference'' denotes the other physical leg. Validity masks
restrict supervision to existing (bus, phase) targets with a defined
nominal angle reference.

\subsubsection{Phase Attribution}

Let $m_{c,t}$ indicate availability of the consumer's
phase attribution voltage observation. Temporal attention uses
this indicator as an additional scoring feature:
\begin{equation}
\beta_{c,t}^{\mathrm{PA}}
=
\operatorname{softmax}_{t}
\left(
a_{\mathrm{PA}}([\mathbf{x}_{c,t}\Vert m_{c,t}])
\right).
\end{equation}
The readout pools projected values and predicts a physical phase:
\begin{align}
\mathbf{z}_{c}^{\mathrm{PA}}
&=
\sum_{t=1}^{T}
\beta_{c,t}^{\mathrm{PA}}f_{\mathrm{PA}}(\mathbf{x}_{c,t}),\\
\mathbf{p}_{c}^{\mathrm{PA}}
&=
\operatorname{softmax}
\left(g_{\mathrm{PA}}(\mathbf{z}_{c}^{\mathrm{PA}})\right).
\end{align}
The output classes are $A,B,C$.
All timesteps participate in the softmax, including those without
a local voltage observation, because spatial message passing can
supply information from other entities.

\subsubsection{Switch state Prediction}

For each line, $m_{e,t}$ indicates whether any phase current
measurement is available. A separate attention scorer produces
\begin{equation}
\beta_{e,t}^{\mathrm{SW}}
=
\operatorname{softmax}_{t}
\left(
a_{\mathrm{SW}}([\mathbf{x}_{e,t}\Vert m_{e,t}])
\right).
\end{equation}
We fuse three complementary summaries:
\begin{equation}
\mathbf{z}_{e}^{\mathrm{SW}}
=
f_{\mathrm{SW}}\left(
\left[
\sum_{t=1}^{T}\beta_{e,t}^{\mathrm{SW}}\mathbf{x}_{e,t}
\;\Vert\;
\frac{1}{T}\sum_{t=1}^{T}\mathbf{x}_{e,t}
\;\Vert\;
\operatorname{max}_{t}\mathbf{x}_{e,t}
\right]\right).
\end{equation}
The maximum is elementwise, and all three summaries use the
complete sequence. The open state probability is
\begin{equation}
p_e^{\mathrm{open}}
=
\sigma\left(g_{\mathrm{SW}}(\mathbf{z}_{e}^{\mathrm{SW}})\right).
\end{equation}
One label is predicted for the window, with supervision restricted
to switch bearing lines.

\subsubsection{Fault Diagnosis}

Fault windows place the event at timestep $T$.
For an input sequence $\mathbf{x}_{1:T}$, the readout constructs
an attended history:
\begin{align}
\gamma_t
&=
\operatorname{softmax}_{t<T}
\left(a_{\mathrm{hist}}(\mathbf{x}_t+\mathbf{p}_t)\right),\\
\mathbf{r}
&=
\sum_{t=1}^{T-1}\gamma_t\mathbf{x}_t,
\end{align}
where $\mathbf{p}_t$ is a learned history position embedding.
It then compares the endpoint with this history:
\begin{equation}
\mathcal{R}(\mathbf{x}_{1:T})
=
f_{\mathrm{event}}\left(
[\mathbf{x}_T\Vert\mathbf{r}
\Vert(\mathbf{x}_T-\mathbf{r})
\Vert|\mathbf{x}_T-\mathbf{r}|]
\right).
\end{equation}

The candidate set is
$\mathcal{C}(\mathcal{G})=
\mathcal{V}_{\mathrm{bus}}\cup
\mathcal{E}_{\mathrm{line}}\cup
\mathcal{E}_{\mathrm{transformer}}$.
Bus candidates use their node sequences directly.
Line and transformer contexts share an affine projection into
the same $d$-dimensional space:
\begin{equation}
\mathbf{q}_{c,t}
=
\begin{cases}
\mathbf{h}_{c,t},
& c\in\mathcal{V}_{\mathrm{bus}},\\
\mathbf{W}_{\mathrm{edge}}\mathbf{x}_{c,t}+\mathbf{b}_{\mathrm{edge}},
& c\in\mathcal{E}_{\mathrm{line}}
\cup\mathcal{E}_{\mathrm{transformer}}.
\end{cases}
\end{equation}
A shared candidate readout and a separately parameterized network
readout produce
\begin{equation}
\mathbf{z}_{c}^{\mathrm{F}}
=
\mathcal{R}_{\mathrm{cand}}(\mathbf{q}_{c,1:T}),
\qquad
\mathbf{z}_{\mathcal{G}}^{\mathrm{F}}
=
\mathcal{R}_{\mathrm{graph}}
(\mathbf{g}^{\mathrm{F}}_{\mathcal{G},1:T}).
\end{equation}
The fault specific network sequence is constructed with whole window
grid summaries, including tap statistics, supplied only to its
final timestep.

All candidate types use the same location scorer:
\begin{equation}
s_c
=
g_{\mathrm{loc}}
([\mathbf{z}_{c}^{\mathrm{F}}
\Vert\mathbf{z}_{\mathcal{G}}^{\mathrm{F}}]),
\qquad
\widehat{c}
=
\arg\max_{c\in\mathcal{C}(\mathcal{G})}s_c.
\end{equation}
The winning candidate determines both equipment type and location.
Because the scorer is shared, its output set follows the current network
rather than a fixed vocabulary of training locations.
Localization supervision applies to applicable faulted samples.

Fault classification combines the network readout with
elementwise maximum evidence across all candidate types:
\begin{equation}
\mathbf{p}^{\mathrm{F}}
=
\operatorname{softmax}\left(
g_{\mathrm{cls}}\left(
\left[
\mathbf{z}_{\mathcal{G}}^{\mathrm{F}}
\;\Vert\;
\operatorname{max}_{c\in\mathcal{C}(\mathcal{G})}
\mathbf{z}_{c}^{\mathrm{F}}
\right]\right)\right).
\end{equation}
The six classes are normal operation, line-to-ground (LG), line-to-line
(LL), double-line-to-ground (LLG), three-phase (LLL), and three-phase-to-ground
(LLLG) faults. Fault detection follows from the distinction between the normal
class and the five fault classes.
The network representation already incorporates mean pooled bus
context, while the candidate maximum retains localized evidence.

The phase, switch, fault class, and location heads each use a
128-dimensional hidden layer with SiLU activation and dropout
$0.1$, followed by an output linear layer. The state decoder is
a single linear layer following the causal temporal refiner.

\subsection{Multi-task Objective and Supervision}
\label{app:multitask_objective}

The joint objective uses unit coefficients for the four task losses:
\begin{equation}
\mathcal{L}
=\mathcal{L}_{\mathrm{SE}}+\mathcal{L}_{\mathrm{PA}}
+\mathcal{L}_{\mathrm{SW}}+\mathcal{L}_{\mathrm{F}}.
\end{equation}
Here, SE, PA, SW, and F denote the state estimation, phase attribution,
switch state inference, and fault diagnosis losses, respectively. Unit coefficients apply after each task's internal scaling and reduction.
They do not imply equal loss magnitudes or gradient contributions.

\paragraph{Reduction over valid targets}
For an entity level task $q$, let $M_{g,u}^{q}$ indicate whether target $u$
in network window $g$ is supervised, and let $\mathcal{B}_q$ contain windows
with at least one supervised target. We use
\begin{equation}
\operatorname{Avg}_{q}(\ell)
=
\frac{1}{|\mathcal{B}_{q}|}
\sum_{g\in\mathcal{B}_{q}}
\frac{\sum_u M_{g,u}^{q}\ell_{g,u}}
     {\sum_u M_{g,u}^{q}}.
\end{equation}

Thus, each applicable network window contributes its mean valid target loss,
preventing larger networks from receiving greater weight solely because they
contain more supervised entities. Windows without valid targets are omitted
from the corresponding average.

\paragraph{State estimation}
Over valid (bus, timestep, phase) targets, we minimize
\begin{equation}
\mathcal{L}_{\mathrm{SE}}
=
\frac{1}{2}\operatorname{Avg}_{\mathrm{SE}}
\left[
\rho\!\left(\lambda_V(\widehat{|V|}^{\mathrm{pu}}-|V|^{\mathrm{pu}})\right)
+
\rho\!\left(\lambda_\theta
\operatorname{wrap}(\widehat{\theta}-\theta)\right)
\right].
\end{equation}
Here $\rho$ is Smooth-$L_1$ with unit transition threshold. Angles are in
radians, and scaling precedes application of $\rho$.
The constants $\lambda_V$ and $\lambda_\theta$ control the magnitude and
angle residual scales, respectively.

\paragraph{Phase attribution and switch state}
Phase attribution uses three-class cross entropy on structurally eligible
consumers with at least one observed phase voltage reading in the window:
\begin{equation}
\mathcal{L}_{\mathrm{PA}}
=
\operatorname{Avg}_{\mathrm{PA}}[-\log p_c(y_c)].
\end{equation}
Switch state prediction uses unweighted binary cross entropy on
switch bearing lines, with $y_e=1$ for OPEN:
\begin{equation}
\mathcal{L}_{\mathrm{SW}}
=
\operatorname{Avg}_{\mathrm{SW}}
[-y_e\log p_e-(1-y_e)\log(1-p_e)].
\end{equation}

\paragraph{Fault diagnosis}
The classification head predicts normal operation or one of five fault types.
With inverse frequency weights $w_k\propto f_k^{-1}$, where $f_k$ is the
configured training corpus frequency, its loss is
\begin{equation}
\mathcal{L}_{\mathrm{cls}}
=
-\frac{\sum_{g\in\mathcal{B}}w_{y_g}\log p_g(y_g)}
       {\sum_{g\in\mathcal{B}}w_{y_g}}.
\end{equation}
For faulted windows with valid location targets $\mathcal{B}_{\mathrm{F}}$,
localization applies softmax
cross entropy over the bus, line, and transformer candidates
$\mathcal{C}(g)$:
\begin{equation}
\mathcal{L}_{\mathrm{loc}}
=
-\frac{1}{|\mathcal{B}_{\mathrm{F}}|}
\sum_{g\in\mathcal{B}_{\mathrm{F}}}
\log
\frac{\exp(s_{g,c_g^\star})}
     {\sum_{c\in\mathcal{C}(g)}\exp(s_{g,c})}.
\end{equation}
We combine the two components when the batch contains a valid fault location target:
\begin{equation}
\mathcal{L}_{\mathrm{F}}
=
\begin{cases}
\frac{1}{2}(\mathcal{L}_{\mathrm{cls}}+\mathcal{L}_{\mathrm{loc}}),
& |\mathcal{B}_{\mathrm{F}}|>0,\\
\mathcal{L}_{\mathrm{cls}}, & \text{otherwise}.
\end{cases}
\end{equation}

\section{Sensor Policies}
\label{app:sensor_policies}

Every electrical case is rendered under 7 named sensor policies. A policy
fixes (i) \emph{which} positions carry sensors (structural placement, drawn
once per network per policy and persisted), and (ii) the noise and missingness
applied to the resulting readings at load time (redrawn from persisted
per sample seeds, so every sample is reproducible).

\paragraph{Coverage levels.}
Coverage fractions are calibrated to utility practice: The six degraded policies apply three fixed tiers under two placement strategies, while \texttt{clean} is the fully observed,
noiseless reference.

\begin{table}[h]
\centering
\begin{tabular}{lccccc}
\toprule
Tier & Bus voltage & DT meter (joint $V{+}I$) & Branch flow & Customer $P$/$Q$ & of which $Q$ \\
\midrule
sparse & 1\% & 60\% & 0.5\% & 10\% & 20\% \\
medium & 5\% & 90\% & 2\% & 30\% & 20\% \\
dense & 15\% & 100\% & 5\% & 70\% & 20\% \\
\texttt{clean} & 100\% & 100\% & 100\% & 100\% & 100\% \\
\bottomrule
\end{tabular}
\caption{Structural coverage per tier. A ``DT meter'' is one physical device:
its terminal bus voltage and transformer edge current channels are always
granted together, never sampled independently. Customer coverage is a binary
AMI/unobserved split. Among AMI customers, reactive power is reported by the
stated fraction. The substation reference voltage and feeder head $P$/$Q$ are
observed under every policy.}
\end{table}

\paragraph{Placement algorithms.}
The \texttt{noisy\_missing\_\{sparse,medium,dense\}} policies place sensors
\emph{uniformly at random} within each eligible pool (plain buses, DT
terminal buses, branches, customers) at the tier's target count. The
\texttt{observability\_ami\_\{sparse,medium,dense\}} policies use the same
target counts but place voltage and branch sensors by {greedy
farthest point ($k$-center) selection on the network's hop distance graph}:
sensors are added iteratively at the bus currently at maximum hop distance
from its nearest existing sensor, minimizing the worst-case
distance to the nearest sensor rather than allowing random clustering. The DT
pool is filled first and the plain bus pass is seeded from it. Customer AMI
assignment is {clustered by serving transformer}: entire DT service
areas receive AMI together, modeling area wide utility rollout programs
rather than an independent coin flip per meter.

\paragraph{Noise and missingness.}
For all policies except \texttt{clean}, each observed reading receives
multiplicative Gaussian noise $x \mapsto x\,(1+\varepsilon)$,
$\varepsilon\!\sim\!\mathcal{N}(0,\sigma^2)$, with $\sigma$ per channel:
voltage magnitude $0.2\%$, voltage angle $0.02^\circ$ (additive), $P$/$Q$
injections $2\%$, currents $1\%$. Regulator tap readback is exact. Each
reading is independently dropped with per channel probability: voltage
$1\%$, injections $5\%$, currents $5\%$. \texttt{clean} sets all noise and
missingness to zero.

\section{Further Experiment Results}
\label{sec:further-res}

\subsection{State Estimation}

\begin{equation}
\arg\min_{\mathbf{x}_t}
\sum_{m\in\mathcal{O}t}
\frac{\left[z_{m,t}-h_m(\mathbf{x}_t;\mathcal{G}_t)\right]^2}
{\sigma_m^2},
\end{equation}

where $\mathcal{O}_t$ is the set of available measurements, $h_m$ is the measurement function, and $\sigma_m$ is the corresponding noise scale. We evaluate voltage magnitudes and angles on the same valid (bus, phase, timestep) entries used for the learned estimators.

\begin{table*}[t]
\centering
\caption{State estimation errors and per-case runtimes for WLS and Mycelium across benchmark networks and sensor policies. Each result is voltage magnitude MAE in volts (per unit) / voltage angle MAE in degrees. $\dagger$ marks WLS divergence, and $\ddagger$ marks urban cases in which WLS did not converge and returned its initialization.}
\label{tab:wls_gnn}
\small
\setlength{\tabcolsep}{4pt}
\renewcommand{\arraystretch}{1.05}
\begin{tabular}{llrrrr}
\toprule
Network & Policy & \multicolumn{2}{c}{WLS} & \multicolumn{2}{c}{Ours} \\
\cmidrule(lr){3-4}\cmidrule(lr){5-6}
& & Result & Time & Result & Time \\
\midrule

\multirow{6}{*}{240\_bus}
% & clean       & 11.4 (.0029) / \textbf{.23} & 18s & \textbf{14.2 (.0025)} / .82 & .4s \\
& nm\_dense   & 34.0 (.0091) / .63 & 7s & \textbf{10.3 (.0022) / .67} & .4s \\
& nm\_medium  & \textbf{11.4 (.0029) / .23} & 7s & 13.8 (.0026) / .81 & .4s \\
& nm\_sparse  & 18.8 (.0049) / \textbf{.35} & 6s & \textbf{13.0 (.0028)} / .80 & .4s \\
& ami\_dense  & 21.6 (.0057) / \textbf{.40} & 7s & \textbf{12.4 (.0024)} / .76 & .4s \\
& ami\_medium & 34.0 (.0091) / .63 & 8s & \textbf{9.9 (.0022) / .65} & .4s \\
& ami\_sparse  & 78.3 (.0149) / \textbf{.62} & 5s & \textbf{10.4 (.0023)} / .66 & .4s \\
\midrule

\multirow{6}{*}{296\_bus}
% & clean       & 58.7 (.0164) / 1.10 & 20s & \textbf{11.1 (.0020) / .61} & .4s \\
& nm\_dense   & 66.1 (.0183) / 1.23 & 10s & \textbf{11.2 (.0020) / .62} & .4s \\
& nm\_medium  & 58.7 (.0164) / 1.10 & 9s & \textbf{10.1 (.0022) / .58} & .4s \\
& nm\_sparse  & 36.1 (.157$^\dagger$) / 4.08$^\dagger$ & 13s
& \textbf{11.0 (.0024) / .67} & .4s \\
& ami\_dense  & 65.7 (.0180) / 1.23 & 9s & \textbf{10.5 (.0020) / .63} & .4s \\
& ami\_medium & 66.1 (.0183) / 1.23 & 8s & \textbf{11.0 (.0021) / .62} & .4s \\
& ami\_sparse  & 68.4 (.0188) / 1.27 & 7s & \textbf{10.4 (.0020) / .65} & .4s \\
\midrule

\multirow{6}{*}{industrial}
% & clean       & 14.7 (.0031) / \textbf{.01} & 8.0m
% & \textbf{9.9} (.0040) / .83 & 1.0s \\
& nm\_dense   & \textbf{19.6} (.0091) / \textbf{.81} & 2.2m
& 20.7 (\textbf{.0056}) / 1.09 & 1.0s \\
& nm\_medium  & \textbf{19.8 (.0085) / .84} & 1.9m
& 38.6 (.0082) / 1.15 & 1.0s \\
& nm\_sparse  & \textbf{12.0 (.0044) / .81} & 9.7m
& 53.7 (.0129) / 1.39 & 1.0s \\
& ami\_dense  & \textbf{5.8 (.0024) / .18} & 5.7m
& 41.5 (.0077) / 1.27 & 1.0s \\
& ami\_medium & \textbf{7.9 (.0041) / .28} & 3.8m
& 49.0 (.0091) / 1.28 & 1.0s \\
& ami\_sparse  & 66.5 (.0114) / 2.77 & 10.1m
& \textbf{57.6 (.0109) / 1.32} & 1.0s \\
\midrule

\multirow{6}{*}{rural}
% & clean       & 3.72 (\textbf{.0045}) / \textbf{.54} & 8.4h
% & \textbf{1.62} (.0058) / .88 & 4.5s \\
& nm\_dense   & \textbf{3.82 (.0049) / .59} & 4.0h
& 4.01 (.0056) / 1.26 & 4.5s \\
& nm\_medium  & \textbf{3.72 (.0045) / .54} & 2.3h
& 4.09 (.0060) / 1.37 & 4.5s \\
& nm\_sparse  & \textbf{3.12 (.0019) / .20} & 1.4h
& 4.37 (.0072) / 1.43 & 4.5s \\
& ami\_dense  & 3.82 (.0049) / \textbf{.59} & 4.0h
& \textbf{1.29 (.0044)} / 1.26 & 4.5s \\
& ami\_medium & 3.82 (.0049) / \textbf{.60} & 2.3h
& \textbf{1.43 (.0049)} / 1.36 & 4.5s \\
& ami\_sparse  & 3.82 (.0049) / \textbf{.59} & 1.3h
& \textbf{1.90} (.0070) / 1.45 & 4.5s \\
\midrule

\multirow{6}{*}{urban}
% & clean$^\ddagger$       & 556.1 (.0517) / 4.59 & 11.9h
% & \textbf{64.4 (.0152) / 4.64} & 18.9s \\
& nm\_dense$^\ddagger$   & 556.4 (.0530) / 4.67 & 6.2h
& \textbf{471.6 (.0453) / 1.97} & 18.9s \\
& nm\_medium$^\ddagger$  & 556.1 (.0517) / 4.59 & 4.2h
& \textbf{543.1 (.0504) / 1.45} & 18.9s \\
& nm\_sparse$^\ddagger$  & 556.4 (.0530) / 4.67 & 2.9h
& 557.6 (.0551) / \textbf{1.50} & 18.9s \\
& ami\_dense$^\ddagger$  & 556.1 (.0517) / 4.59 & 6.1h
& \textbf{72.4 (.0115) / 4.53} & 18.9s \\
& ami\_medium$^\ddagger$ & 556.4 (.0530) / 4.67 & 4.0h
& \textbf{68.8 (.0169)} / 4.71 & 18.9s \\
& ami\_sparse$^\ddagger$ & 556.1 (.0517) / 4.59 & 3.1h
& \textbf{451.1 (.0472) / 2.16} & 18.9s \\
\bottomrule
\end{tabular}
\end{table*}

WLS receives the exact per phase network model, operational switch states, timestep specific tap positions, measurement mapping, and simulator noise scales. It also uses known zero injection constraints and a substation angle reference. It is therefore an oracle-assisted reference, rather than a baseline operating under Mycelium's information constraints. Table~\ref{tab:wls_gnn} reports results and per-case runtimes across the five frozen benchmark networks. WLS performs well in several industrial and rural settings, whereas Mycelium obtains lower errors in many of the Iowa network cases. WLS does not converge on the urban network. Its reported urban predictions are the solver initialization and should not be interpreted as converged WLS estimates.

\paragraph{Privileged Information Required by WLS. }Our WLS baseline receives simulator derived information beyond the measurements available to Mycelium:

\begin{enumerate}
\item \textbf{Exact network model:} per phase connectivity, impedances, transformers, regulators, shunts, and low voltage secondaries. In a 240-bus diagnostic, an erroneous regulator shunt term increased voltage error from approximately \(11\) to \(105\mathrm{V}\).
\item \textbf{True operating topology:} the actual switch states used to construct the admittance matrix.
\item \textbf{True tap positions:} transformer and regulator taps at every timestep, including those used for initialization.
\item \textbf{Correct mappings and zero injections:} bus and phase assignments for measurements, plus known zero injection buses imposed as hard Hachtel constraints.
\item \textbf{True noise scales:} weights derived from the measurement noise distributions used by the simulator.
\item \textbf{Substation angle reference:} an available source angle at every timestep, including under sparse sensor policies.
\end{enumerate}

These assumptions make WLS an oracle-assisted benchmark: its results reflect state estimation with exact topology, device settings, phase mapping, and noise statistics. Mycelium receives less operating information and must infer some of these quantities from sparse measurements. Performance comparisons should therefore be interpreted in light of these different inputs.

Table~\ref{tab:se_per_policy} compares Mycelium and CANOS under each sensor policy. Mycelium has lower per-unit magnitude error under five of the seven policies and lower angle error under six, while CANOS has lower magnitude error in volts under every policy. Because networks have different voltage bases, the per-unit and volt metrics can rank models differently. A breakdown by voltage level would be needed to identify the source of this reversal.

\begin{table*}[t]
\centering
\caption{Per policy state estimation on the benchmark (all five networks,
with fault, 24 timesteps).}
\label{tab:se_per_policy}
\small
\setlength{\tabcolsep}{4pt}
\renewcommand{\arraystretch}{1.05}
\begin{tabular}{lrrrrrrrr}
\toprule
& \multicolumn{4}{c}{CANOS} & \multicolumn{4}{c}{Ours} \\
\cmidrule(lr){2-5}\cmidrule(lr){6-9}
Policy & MAE (V) & MAE (p.u.) & MAE ($^\circ$) & MAPE
       & MAE (V) & MAE (p.u.) & MAE ($^\circ$) & MAPE \\
\midrule
clean       & \textbf{1.98} & \textbf{.0017} & \textbf{.63} & \textbf{0.35\%} & 18.71 & .0079 & 1.77 & 5.43\% \\
nm\_dense   & \textbf{6.21} & \textbf{.0100} & 3.83 & \textbf{4.97\%} & 116.36 & .0145 & \textbf{1.37} & 5.31\% \\
nm\_medium  & \textbf{10.24} & .0205 & 4.15 & 6.28\% & 135.07 & \textbf{.0166} & \textbf{1.32} & \textbf{5.38\%} \\
nm\_sparse  & \textbf{12.96} & .0276 & 3.66 & 7.31\% & 140.04 & \textbf{.0186} & \textbf{1.36} & \textbf{5.58\%} \\
ami\_dense  & \textbf{8.74} & .0113 & 3.53 & 5.73\% & 22.08 & \textbf{.0063} & \textbf{1.99} & \textbf{5.25\%} \\
ami\_medium & \textbf{12.17} & .0230 & 3.50 & 6.85\% & 23.48 & \textbf{.0080} & \textbf{2.09} & \textbf{5.40\%} \\
ami\_sparse & \textbf{14.21} & .0288 & 3.47 & 7.65\% & 113.26 & \textbf{.0167} & \textbf{1.56} & \textbf{5.63\%} \\
\midrule
\textbf{Aggregate} & \textbf{9.50} & .0175 & 3.25 & 5.59\% & 81.28 & \textbf{.0127} & \textbf{1.64} & \textbf{5.42\%} \\
\bottomrule
\end{tabular}
\end{table*}

\subsection{Phase Attribution}
\label{app:knn_phase}

We also compare Mycelium with a correlation-based phase matcher. For each eligible customer, the matcher compares its 24-step meter profile with the profiles for the three phases at its serving transformer and assigns the phase with the highest correlation. We evaluate variants using transformer active power ($P_{\mathrm{xfmr}}$) and active and reactive power jointly ($P{+}Q_{\mathrm{xfmr}}$). All methods are scored on the same eligible customer predictions.

Table~\ref{tab:knn_agg} shows that the 20,000-bus Mycelium model exceeds both reported power profile variants on the test split and frozen benchmark. The 10,000-bus model closely matches the $P_{\mathrm{xfmr}}$ variant numerically (0.5713 versus 0.5741 on test and 0.4001 versus 0.3999 on the benchmark), but trails the $P{+}Q_{\mathrm{xfmr}}$ variant on the benchmark (0.4001 versus 0.4200). The two matching variants also reverse rank across splits, showing that their relative performance depends on the evaluation distribution and choice of signals.

\begin{table}[h]
\centering
\caption{Pooled phase attribution accuracy. Populations are matched to the
models': $n{=}1{,}315{,}577$ on the benchmark, $n{=}2{,}862$ on the test split.}
\label{tab:knn_agg}
\begin{tabular}{lcc}
\toprule
\textbf{Method} & \textbf{Benchmark} & \textbf{Test} \\
\midrule
correlation-based ($P_{\mathrm{xfmr}}$)     & 0.3999 & 0.5741 \\
correlation-based ($P{+}Q_{\mathrm{xfmr}}$) & 0.4200 & 0.3637 \\
\midrule
Mycelium 10k                       & 0.4001 & 0.5713 \\
Mycelium 20k                             & \textbf{0.4387} & \textbf{0.6150} \\
\bottomrule
\end{tabular}
\end{table}

No single number characterizes the baseline. The two variants swap rank between splits -- $P{+}Q$ is stronger on the benchmark, $P$ on the test split. We report
both rather than select one, since choosing a single variant would move the comparison by more than $0.2$ in either direction.

\subsection{Fault Localization}
\label{app:fault_localization}

We compare Mycelium with STGATv2 on the frozen benchmark, reporting
results by sensor policy and network. Localization is evaluated on
faulted samples over each network's full candidate set of buses,
lines, and transformers. Loc.\ denotes exact localization accuracy,
Equip.\ denotes equipment type accuracy, and Hop@$k$ measures
localization within $k$ hops. Class reports fault class accuracy.

Table~\ref{tab:fault_by_policy} shows that Mycelium achieves higher
exact localization accuracy under six sensor policies and ties
STGATv2 under dense AMI sensing. It also improves equipment type
accuracy and all hop-based metrics under every policy. However,
performance remains sensitive to observation quality: Mycelium's
exact localization accuracy falls from 0.500 under clean measurements
to 0.102 under sparse, noisy, and missing measurements.

\begin{table}[t]
\centering
\caption{Fault diagnosis by sensor policy on the frozen benchmark,
with 108 faulted samples per policy for localization.
Det.\ F1 is the fault versus normal detection F1 obtained by collapsing the
fault class matrix. Class$_5$ is the 5-way type accuracy among genuinely
faulted samples. Higher is better for all metrics. Best comparable values
within each policy are bold, including ties.}
\label{tab:fault_by_policy}
\resizebox{0.7\linewidth}{!}{%
\begin{tabular}{llcccccc}
\toprule
Policy & Model & Det.\ F1 & class$_5$ & Loc. & Hop@1 & Hop@2 & Hop@3 \\
\midrule
\multirow{2}{*}{clean}
& STGATv2 & - & - & 0.213 & 0.250 & 0.269 & 0.306 \\
& Mycelium & 1.000 & 0.704 & \textbf{0.500}
& \textbf{0.537} & \textbf{0.583} & \textbf{0.583} \\
\midrule
\multirow{2}{*}{nm\_dense}
& STGATv2 & - & - & 0.130 & 0.213 & 0.222 & 0.306 \\
& Mycelium & 1.000 & 0.704 & \textbf{0.185}
& \textbf{0.315} & \textbf{0.370} & \textbf{0.435} \\
\midrule
\multirow{2}{*}{nm\_medium}
& STGATv2 & - & - & 0.130 & 0.231 & 0.241 & 0.296 \\
& Mycelium & 1.000 & 0.769 & \textbf{0.185}
& \textbf{0.324} & \textbf{0.361} & \textbf{0.417} \\
\midrule
\multirow{2}{*}{nm\_sparse}
& STGATv2 & - & - & 0.065 & 0.157 & 0.157 & 0.222 \\
& Mycelium & 1.000 & 0.769 & \textbf{0.102}
& \textbf{0.259} & \textbf{0.296} & \textbf{0.361} \\
\midrule
\multirow{2}{*}{ami\_dense}
& STGATv2 & - & - & \textbf{0.111} & 0.194 & 0.194 & 0.259 \\
& Mycelium & 1.000 & 0.722 & \textbf{0.111}
& \textbf{0.213} & \textbf{0.259} & \textbf{0.324} \\
\midrule
\multirow{2}{*}{ami\_medium}
& STGATv2 & - & - & 0.102 & 0.176 & 0.185 & 0.241 \\
& Mycelium & 1.000 & 0.732 & \textbf{0.148}
& \textbf{0.315} & \textbf{0.343} & \textbf{0.407} \\
\midrule
\multirow{2}{*}{ami\_sparse}
& STGATv2 & - & - & 0.074 & 0.139 & 0.185 & 0.250 \\
& Mycelium & 1.000 & 0.648 & \textbf{0.139}
& \textbf{0.259} & \textbf{0.278} & \textbf{0.352} \\
\bottomrule
\end{tabular}%
}
\end{table}

The network-level breakdown in Table~\ref{tab:fault_by_network}
reveals substantial variation across networks. Mycelium improves
exact localization on both Iowa networks and the industrial and
rural networks. STGATv2 performs better on the urban network,
with exact accuracy of 0.222 versus 0.079, and retains a small
advantage in rural Hop@1 and Hop@2. Thus, the aggregate gains
do not imply uniformly better localization across networks.

% \begin{table}[t]
% \centering
% \caption{Fault diagnosis by network on the frozen benchmark,
% pooled across sensor policies. $n$ counts faulted samples used
% for localization. Higher is better for all metrics. Best
% comparable values within each network are bold.}
% \label{tab:fault_by_network}
% \resizebox{0.7\linewidth}{!}{%
% \begin{tabular}{llrcccccc}
% \toprule
% Network & Model & $n$ & Class & Loc. & Equip. & Hop@1 & Hop@2 & Hop@3 \\
% \midrule
% \multirow{2}{*}{240\_bus}
% & STGATv2 & 189 & - & 0.159 & 0.460 & 0.159 & 0.159 & 0.159 \\
% & Mycelium & 189 & 0.757 & \textbf{0.312} & \textbf{0.524}
% & \textbf{0.455} & \textbf{0.481} & \textbf{0.545} \\
% \midrule
% \multirow{2}{*}{296\_bus}
% & STGATv2 & 189 & - & 0.005 & 0.333 & 0.090 & 0.138 & 0.201 \\
% & Mycelium & 189 & 0.846 & \textbf{0.116} & \textbf{0.407}
% & \textbf{0.270} & \textbf{0.307} & \textbf{0.307} \\
% \midrule
% \multirow{2}{*}{industrial}
% & STGATv2 & 126 & - & 0.127 & 0.452 & 0.127 & 0.127 & 0.254 \\
% & Mycelium & 126 & 0.842 & \textbf{0.214} & \textbf{0.730}
% & \textbf{0.238} & \textbf{0.310} & \textbf{0.381} \\
% \midrule
% \multirow{2}{*}{rural}
% & STGATv2 & 189 & - & 0.148 & 0.349
% & \textbf{0.339} & \textbf{0.344} & 0.439 \\
% & Mycelium & 189 & 0.842 & \textbf{0.185} & \textbf{0.423}
% & 0.317 & 0.333 & \textbf{0.444} \\
% \midrule
% \multirow{2}{*}{urban}
% & STGATv2 & 63 & - & \textbf{0.222} & 0.317
% & \textbf{0.317} & \textbf{0.317} & \textbf{0.317} \\
% & Mycelium & 63 & 0.486 & 0.079 & \textbf{0.397}
% & 0.206 & 0.286 & 0.286 \\
% \bottomrule
% \end{tabular}%
% }
% \end{table}

\begin{table}[t]
\centering
\caption{Fault diagnosis by network on the frozen benchmark,
pooled across sensor policies. $n$ counts faulted samples used
for localization. Det.\ F1 is the fault versus normal detection F1 obtained by
collapsing the fault class matrix. Class$_5$ is the 5-way type accuracy among
genuinely faulted samples. Higher is better for all metrics. Best
comparable values within each network are bold.}
\label{tab:fault_by_network}
\resizebox{0.7\linewidth}{!}{%
\begin{tabular}{llrcccccc}
\toprule
Network & Model & $n$ & Det.\ F1 & class$_5$ & Loc. & Hop@1 & Hop@2 & Hop@3 \\
\midrule
\multirow{2}{*}{240\_bus}
& STGATv2 & 189 & - & - & 0.159 & 0.159 & 0.159 & 0.159 \\
& Mycelium & 189 & 1.000 & 0.667 & \textbf{0.312}
& \textbf{0.455} & \textbf{0.481} & \textbf{0.545} \\
\midrule
\multirow{2}{*}{296\_bus}
& STGATv2 & 189 & - & - & 0.005 & 0.090 & 0.138 & 0.201 \\
& Mycelium & 189 & 1.000 & 0.788 & \textbf{0.116}
& \textbf{0.270} & \textbf{0.307} & \textbf{0.307} \\
\midrule
\multirow{2}{*}{industrial}
& STGATv2 & 126 & - & - & 0.127 & 0.127 & 0.127 & 0.254 \\
& Mycelium & 126 & 1.000 & 0.754 & \textbf{0.214}
& \textbf{0.238} & \textbf{0.310} & \textbf{0.381} \\
\midrule
\multirow{2}{*}{rural}
& STGATv2 & 189 & - & - & 0.148
& \textbf{0.339} & \textbf{0.344} & 0.439 \\
& Mycelium & 189 & 1.000 & 0.783 & \textbf{0.185}
& 0.317 & 0.333 & \textbf{0.444} \\
\midrule
\multirow{2}{*}{urban}
& STGATv2 & 63 & - & - & \textbf{0.222}
& \textbf{0.317} & \textbf{0.317} & \textbf{0.317} \\
& Mycelium & 63 & 1.000 & 0.429 & 0.079
& 0.206 & 0.286 & 0.286 \\
\bottomrule
\end{tabular}%
}
\end{table}

\subsection{Per-network Benchmark Results}
\label{app:per-network}

Table~\ref{tab:tenk_final_per_network} reveals variation obscured by pooled metrics. The urban network has the largest per-unit and absolute voltage errors, while the industrial network has a MAPE of 60.92\%, compared with less than 1.5\% on every other benchmark network. This raises network macro MAPE to 12.86\%, versus 5.43\% pooled. Exact fault localization accuracy ranges from 0.0794 on urban to 0.3122 on 240\_bus. Phase attribution is evaluated only on rural and urban, where accuracy is 0.4898 and 0.3782, respectively.

The industrial MAPE is highly sensitive to near-zero targets. Although buses without a valid voltage base are excluded, some de-energized buses retain a positive base and enter the metric with true magnitudes as low as $1.6\times10^{-4}$\,p.u. Dividing absolute errors by these targets can produce extreme percentage errors. The industrial network's per-unit MAE is 0.0097. Its median and 90th percentile absolute percentage errors are 0.39\% and 0.86\%, and its MAPE after excluding the worst decile is 0.56\%. We therefore emphasize per-unit MAE for voltage magnitude and interpret MAPE cautiously.

\begin{table}[t]
\centering
\caption{Per-network performance of Mycelium 10k on the frozen benchmark. Network macro averages weight networks equally within each metric's evaluation population. Pooled metrics aggregate predictions across networks. Sample counts $n$ refer to evaluation windows. Phase attribution has eligible targets only on the rural and urban networks. Superscripts $k$ indicate the number of networks included when fewer than five contribute. All results use the same checkpoint and evaluation run.}
\label{tab:tenk_final_per_network}
\resizebox{\textwidth}{!}{%
\begin{tabular}{llcccccccccccccc}
\toprule
& & \multicolumn{4}{c}{State estimation}
& Phase & \multicolumn{3}{c}{Switch} & \multicolumn{6}{c}{Fault} \\
\cmidrule(lr){3-6}\cmidrule(lr){7-7}
\cmidrule(lr){8-10}\cmidrule(lr){11-16}
Split & Configuration
& p.u. $\downarrow$ & $v_{\mathrm{mag}}$ (V) $\downarrow$
& $v_{\mathrm{ang}}$ (deg) $\downarrow$ & MAPE (\%) $\downarrow$
& Acc. $\uparrow$ & Prec. $\uparrow$ & Rec. $\uparrow$ & F1 $\uparrow$
& Det. F1 $\uparrow$ & class$_5$ $\uparrow$ & Loc. $\uparrow$
& Hop@1 $\uparrow$ & Hop@2 $\uparrow$ & Hop@3 $\uparrow$ \\
\midrule
\multirow{7}{*}{Benchmark}
& 240\_bus ($n{=}259$)    & 0.0039 & 18.26  & 0.870 & 0.63  & --     & --     & --     & --     & 1.000 & 0.6667 & 0.3122 & 0.4550 & 0.4815 & 0.5450 \\
& 296\_bus ($n{=}259$)    & 0.0034 & 18.32  & 0.742 & 0.76  & --     & --     & --     & --     & 1.000 & 0.7884 & 0.1164 & 0.2698 & 0.3069 & 0.3069 \\
& industrial ($n{=}196$)  & 0.0097 & 40.66  & 1.213 & 60.92 & --     & 0.9974 & 0.5955 & 0.7458 & 1.000 & 0.7540 & 0.2143 & 0.2381 & 0.3095 & 0.3810 \\
& rural ($n{=}259$)       & 0.0059 & 2.70   & 1.286 & 0.59  & 0.4898 & 0.9644 & 0.9564 & 0.9604 & 1.000 & 0.7831 & 0.1852 & 0.3175 & 0.3333 & 0.4444 \\
& urban ($n{=}70$)        & 0.0342 & 318.20 & 2.971 & 1.40  & 0.3782 & 0.9777 & 0.9714 & 0.9746 & 1.000 & 0.4286 & 0.0794 & 0.2063 & 0.2857 & 0.2857 \\
\cmidrule(lr){2-16}
& Network macro           & 0.0114 & 79.63  & 1.416 & 12.86 & 0.4340\textsuperscript{$k$=2} & 0.9798\textsuperscript{$k$=3} & 0.8411\textsuperscript{$k$=3} & 0.8936\textsuperscript{$k$=3} & 1.000 & 0.6841 & 0.1815 & 0.2974 & 0.3434 & 0.3926 \\
& Pooled ($n{=}1{,}043$)  & 0.0127 & 81.28  & 1.636 & 5.43  & 0.4001 & 0.9553 & 0.7198 & 0.8210 & 1.000 & 0.7209 & 0.1958 & 0.3175 & 0.3558 & 0.4114 \\
\bottomrule
\end{tabular}}
\end{table}

\section{Training and Memory Optimization}
\label{app:training_memory}

\paragraph{Hardware and precision.}
We train on eight NVIDIA A100 GPUs, each with 80\,GB of memory, using FP32 throughout, without mixed precision.

\paragraph{Activation checkpointing.}
The backbone preserves separate hidden states for all $T=24$ timesteps across eight message passing layers, making message passing the dominant source of activation memory. We checkpoint each layer using \texttt{torch.utils.checkpoint.checkpoint} with \texttt{use\_reentrant=False} and \texttt{preserve\_rng\_state=True}. This reduces memory usage by recomputing intermediate activations as needed during backpropagation, allowing the configured per GPU batch to fit within 80\,GB.

Preserving the random number generator state reproduces dropout masks during recomputation. Checkpointing adds no parameters and does not change the training objective or state dictionary structure. It trades additional computation during backpropagation for reduced activation storage and is disabled during inference.

\end{document}